\documentclass[journal]{IEEEtran}
\ifCLASSINFOpdf

\else

\fi
\usepackage{algorithm} 		
\usepackage{algorithmic} 	
\usepackage{bm}
\usepackage{booktabs}
\usepackage{graphicx}
\usepackage{amsmath}
\usepackage{cite}
\usepackage{amssymb}
\usepackage{multirow}

\usepackage{epstopdf}
\usepackage{exscale}
\usepackage{relsize}
\usepackage{color}
\usepackage{graphicx}
\usepackage{subfigure}
\usepackage[table]{xcolor}
\usepackage{amssymb,amsmath}
\usepackage{cases}
\usepackage{longtable}
\usepackage{rotating}
\usepackage{multirow}
 \usepackage{siunitx}
\usepackage{epstopdf}
\usepackage{cite}
\usepackage{float}
\usepackage[marginal]{footmisc}
\usepackage{stfloats}
\usepackage[colorlinks,
            linkcolor=black,
            anchorcolor=black,
            urlcolor=black,
            citecolor=black]{hyperref}
\usepackage[hyphenbreaks]{breakurl}

\newtheorem{theorem}{Theorem}

\newtheorem{proof}{Proof}
\newtheorem{definition}{Definition}

\input epsf
\begin{document}
\title{Learner Referral for Cost-Effective Federated Learning Over Hierarchical IoT Networks}
\author{Yulan Gao, \IEEEmembership{~Member, IEEE,}
Ziqiang Ye, \IEEEmembership{}
Yue Xiao, \IEEEmembership{~Member, IEEE,}
and Wei Xiang, \IEEEmembership{~Senior Member, IEEE,}
\thanks{Y. Gao, Z. Ye, and Y. Xiao are with the National Key Laboratory of Science and Technology on Communications, University of Electronic Science and Technology of China (UESTC), Chengdu 611731, China (e-mail: yulangaomath@163.com; yysxiaoyu@hotmail.com; xiaoyue@uestc.edu.cn).}
\thanks{W. Xiang is with the School of Engineering and Mathematical Sciences, La Trobe University, Victoria 3086, Australia (e-mail: W.Xiang@latrobe.edu.au).}
}
\markboth{IEEE Internet of Things Journal}
{Shell \MakeLowercase{\textit{et al.}}: }
\maketitle
\begin{abstract}
The paradigm of federated learning (FL) to address data privacy concerns by locally training parameters on resource-constrained clients in a distributed manner has garnered significant attention.
Nonetheless, FL is not applicable when not all clients within the coverage of the FL server are registered with the FL network.
To bridge this gap, this paper proposes joint learner referral aided federated client selection (LRef-FedCS), along with communications and computing resource scheduling, and local model accuracy optimization (LMAO) methods.
These methods are designed to minimize the cost incurred by the worst-case participant and ensure the long-term fairness of FL in hierarchical Internet of Things (HieIoT) networks.
Utilizing the Lyapunov optimization technique, we reformulate the original problem into a stepwise joint optimization problem (JOP).
Subsequently, to tackle the mixed-integer non-convex JOP, we separatively and iteratively address LRef-FedCS and LMAO through the centralized method and self-adaptive global best harmony search (SGHS) algorithm, respectively.
To enhance scalability, we further propose a distributed LRef-FedCS approach based on a matching game to replace the  centralized method described above.
Numerical simulations and experimental results on the MNIST/CIFAR-10 datasets demonstrate that our proposed LRef-FedCS approach could achieve a good balance between pursuing high global accuracy and reducing cost.
\end{abstract}
\IEEEpeerreviewmaketitle

\section{Introduction}
\IEEEPARstart{F}{ederated} learning (FL) is a transformative computational stategy, conceived within the vast landscape of artificial intelligence (AI), to accomodate the intensifying need for data privacy.
This inventive framework is predicated on the essential concept of computation decentralization, thereby shifting the computational workload from the central server onto a network of local devices ({\em a.k.a.} clients and workers) \cite{yang2019federated}.
By doing so, privacy preservation is ensured for data owners, as FL avoids transmit transporting end devices' raw data to the central server.
From an application standpoint, given the inevitable proliferation of Internet of Things (IoT) devices (e.g., mobile phone, iPad, laptop, vehicles, etc.), we are witnessing a measured yet steady ascent of collaborative training parameters \cite{konevcny2016federated,mcmahan2017communication}.
For instance, UberEATs (\texttt{\url{https://en.wikipedia.org/wiki/Uber_Eats}}), launched by Uber in 2014, has adopted the FL paradigm  to improve on-time delivery while simultaneously protecting data privacy of  both customers and delivery partners \cite{liang2018toward}.

For FL in wireless IoT networks, unreliable date is likely to be uploaded intentionally or unintentionally by mobile clients \cite{kang2019incentive, kang2020reliable}, resulting in a low learning quality.
So, how to design trustworthy and reliable federated client selection (FedCS) strategies is a challenge in wireless FL systems.
In addition, the model quality relies on frequent communications between the FL server and clients, which leads to another challenge in communications and computation.
Accordingly, reaping the potential benefits of wireless FL in practice requires meeting the aforementioned challenges in FedCS and diverse resource management.
Among the early iconic contributions in this area, the authors of \cite{chen2020convergence,chen2020joint} proposed a probabilistic FedCS approach to minimize the training loss while guaranteeing model fast convergence.
The authors of \cite{tran2019federated} used the Lagrangian dual framework to investigate the weighted sum of energy consumption and time cost (WSET)-based minimization problem from the perspective of joint transmit power and computation resources  allocation as well as local model accuracy optimization (LMAO).
The aforementioned  studies for FL in wireless IoT networks assume that all participating clients are willing to learn collaboratively and continuously.
In fact, there are various reasons such that mobile clients are reluctant to contribute any learning resources to participate in the training process, e.g., limited CPU capability, battery, and their own requirements/schedule.
Therefore, incentive mechanisms which incentivize the  trustworthy and cost-effective learning behavior become a  catalyst for making wireless FL a reality.

Nowadays, social attributes have become a vital factor for realizing trustworthy multi-resource sharing thanks to the success of Social Internet of Things (SIoT) \cite{nitti2016exploiting}, where social relationships can be built among mobile users via online social networks (e.g., Twitter, Facebook, and WeChat, etc.).
The knowledge of mutual social relationships among IoT devices can be beneficial to trustworthy and reliable FedCS, thereby mitigating  malicious attacks throughout the learning process.
Discussions on exploiting SIoT features in the FL paradigm have gradually attracted the attention of researchers in recent years.
Most research efforts to date have focused on clustering the FL framework by leveraging the social relationships among IoT devices for minimizing the loss function \cite{khan2021socially,yin2021privacy,cheng2021dynamic}.
More recently, the authors of \cite{lin2021friend} used social relationships among mobile clients to design trustworthy and efficient wireless federated edge learning.
These investigations, however, are based on the premise that all clients (i.e., any client covered by the FL server) are known to the FL server.
In this framework, the fundamental assumption is that all clients could fully and spontaneously participate in the complete training process, without taking into account the time-varying states and schedules of clients in practice.
Motivated by this, we develop socially-driven trustworthy and cost-effective leaner referral mechanism aided (LRef)-FedCS policies for FL in hierarchical (Hie)IoT networks.
To accentuate the role of social trust, we introduce a novel hierarchical dimension within the HieIoT network, founded on the establishment of a social relationship between the client and the FL server.
Consequently, we categorize all clients under the coverage of the FL server into two distinct types, namely Registered FL clients (RC) and Unregistered FL clients (UnRC), based on this relationship's presence or absence.

The main goal of this paper is to design efficient joint LRef-FedCS and LMAO schemes for the time-average WSET minimization and fairness guarantees in the HeiIoT FL network, where RCs will recommend trustworthy UnRCs as temporary learners in some training rounds.
Specifically, the long-term fairness metric adopted in this work is defined from the perspective of the RCs, which can effectively incentivize RCs to perform the learner referral process.
Drawing on the  Lyapunov optimization technique, we transform the complex time-coupled WSET minimization problem into a step-by-step online joint optimization problem (JOP).
Notably, a practical and challenging scenario is considered, where all RCs and UnRCs can move around within the area covered by the server and cannot continuously participate in the complete training process.
The novelty and contributions of this paper are summarized as follows:
\begin{itemize}
\item  We define two participation modes of RCs according to their current status, dubbed the direct participation mode (DParM) and the learner referral mode (LRefM), respectively.
We further classify the LRefM into the Partial- and the Full-LRefM based on the current requirements of the recommended UnRCs.
The main thrust of introducing the above definitions is devoted to facilitating the LRef-FedCS and realizing on-demand and rational scheduling of clients' surplus resources.

\item  To solve the step-by-step online JOP, we propose two algorithms by employing the centralized search, matching game, and self-adaptive global best harmony search (SGHS) algorithm.
We analyse the stability and convergence of the matching game based LRef-FedCS approach.
The complexity of the matching game based LRef-FedCS approach is upper bounded by $MN$, where $M$ and $N$ are the number of RCs and UnRCs, respectively.

\item We run extensive simulations to assess performance of the LRef-FedCS methods for FL in the HieIoT network.
Results show that the solution tightness behavior is sensitive to the balance parameter of Lyapunov, and the distributed method can perform closely to the optimal solution with only $2.33\%$ bias within $10$ iterations.
Moreover, the results of experiments on the MNIST/CIFAR-10 datasets show that the proposed LRef-FedCS scheme strikes a good balance between pursuing high-quality of FL training and reducing system cost.
\end{itemize}
The rest of this paper is organized as follows.
The related work is presented in Sec. II.
The preliminary of FL and system model are given in Sec. III and Sec. IV, respectively.
The problem formulation and the online algorithm design  are respectively presented in Sec. V and Sec. VI.
Sec. VII introduces an distributed LRef-FedCS method based on matching game theory, and Sec. VIII reports the results of simulations and experiments.
Finally, conclusions are presented in Sec. IX.

{\bf\em Notation}: The set of real numbers is denoted by $\mathbb R$.
Vectors and matrices are denoted by bold lowercase and uppercase letters, respectively.
Given an $M\times N$ matrix $\pmb\Lambda$, the the $m\text{-th}$ row is represented by $\pmb\Lambda_{m,:}[ ]\in{\mathbb R}^{1\times N}$.
Conversely,  $\pmb\Lambda_{-m,:}[ ]$ denotes the matrix $\pmb\Lambda$ with the $m\text{-th}$ row removed, encapsulating the remaining elements.
The horizontal concatenation of sets ${\cal M}$ and ${\cal N}$ is denoted by $[{\cal M}~{\cal N}]$.
The expectation of a vector $\boldsymbol x$ and the cardinality of a set ${\cal X}$ are denoted by ${\mathbb E}\{\boldsymbol x\}$ and $|\cal X|$, respectively.

\section{Related Work}


The framework of reliable client selection and scheduling round-by-round focuses mainly on the long-term benefits, and it has emerged as a prominent area of focus in the domain of FL research within wireless IoT networks.
An established theoretical tool for problems of this kind is provided by Lyapunov concept-based optimization theory \cite{neely2012dynamic,neely2013}.
Along this line, most works jointly select energy-efficient clients and schedule resources in different training rounds to achieve fairness or maximize utility \cite{xu2020client,huang2020efficiency,zhu2022online}.
Specifically,  the authors of \cite{xu2020client} proposed an algorithm that only utilizes currently available wireless channel information for joint client selection and bandwidth allocation to meet a preset long-term performance requirement.
They further studied the fairness-guaranteed client selection issue for FL in edge networks \cite{huang2020efficiency}.
In \cite{zhu2022online }, the training time minimization problem for the adaptive client selection design in asynchronous FL was investigated, where the client availability and long-term fairness were considered.
It worth mentioning that the Lyapunov concept-based algorithm is particularly cognizant of time/slot coupled constraints and objectives, i.e., long-term utilities and constraints.
Also incentive mechanisms have become indispensible for FL to sustain long-term joint client/learner selection and resource scheduling.
The authors of \cite{pang2022incentive} formulated a procurement auction to model and analyze the competitive cost minimization problem via selecting and scheduling clients in different training rounds.
For critical energy infrastructure systems, the work in \cite{lu2021toward} developed a multidimensional contract framework to incentivize clients to participate in asynchronous FL for learning accuracy maximization.
More recently, the authors of \cite{zhang2022enabling} used repeated games in the infinite time horizon to model and analyze  clients' interactions in the long-term cross-silo FL process.
Their proposed framework is able to reduce the number of free riders and increase the amount of local data for model training.

The integration of social network features with FL has the potential to promote  trustworthy and reliable FedCS in FL networks.
A few early studies presented relevant solutions.
In \cite{khan2020self}, a clustering algorithm based on social awareness-aided cluster head selection was provided to address the global training time minimization problem.
The proposed framework enables reliable self-organizing FL, where both social relationships and computational resources of the devices are take into account in selecting the cluster head.
A framework of discentralized FL on single-sided trust social networks was proposed in \cite{he2019central}.
Their simulation results and mathematical analysis on rigorous regret show the benefit of communicating with trusted users in the FL network.
Furthermore, a social effect based incentive mechanism for federated edge learning was designed to defend against malicious or inactive learners \cite{lin2021friend}.
In addition, this social federated edge learning framework establishes a trust model via indirect social incentives.
Different from the above studies, we propose joint LRef-FedCS and diverse resource scheduling schemes in a more realistic scenario,  where not all clients are registered at the FL platform (i.e., not all clients are known to the FL server).
Moreover,  we also leverage the social relationships among all clients (RCs and UnRCs) to facilitate the trustworthy LRef-FedCS process in a practical scenario, where all RCs and UnRCs can freely move around within the area covered by the server.
\begin{figure}
\begin{minipage}[t]{0.45\textwidth}
\centering
\includegraphics[width=1\textwidth]{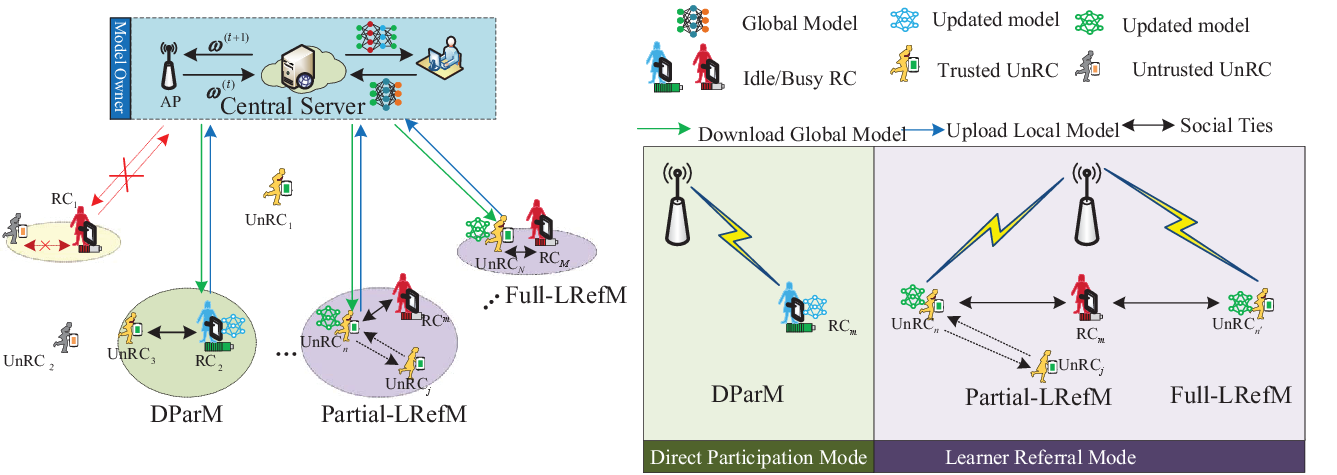}
\caption{System architecture and participation modes of RCs and UnRCs within the coverage of the FL server in the HieIoT network.  }
\label{fig:1}
\end{minipage}
\end{figure}
\section{Preliminary of Federated Learning}
The learning process of FL minimizes the loss function through frequent communication between FL server and learners to update model and gradient until the model converges \cite{tran2019federated}.
Each learner iteratively computes the local model and the gradient until a local accuracy $0\leq \theta\leq 1$ is achieved and upload them to the FL server.
Then, all collected local parameters and gradients are aggregated to generate a global model $\pmb\omega$.
When a specific global accuracy $0\leq \epsilon\leq 1$ is reached, the entire training process is terminated.
Upon achieving an global accuracy level $\epsilon,$ a number of training iterations are required, causing interaction between learners and FL server.
As mentioned in \cite{tran2019federated, ma2017distributed}, the upper bound of training rounds is closely related to the local accuracy $\theta$ and the global model accuracy $\epsilon$, which is specifically expressed as $T_g(\epsilon, \theta)={\cal O}(\log(1/\epsilon))/(1-\theta).$
For a fixed global accuracy $\epsilon$, so the upper bound of training rounds  can be normalized to $T_g=1/(1-\theta).$
Likewise, the upper bound of local iterations is normalized to $\log(1/\theta)$\footnote{Here, $\theta$ and $\epsilon$ are the gradient difference between two successive iterations of loss function on local and global models, respectively.}.

\section{System Model}
We consider an FL system in a HieIoT network consisting of a set ${\cal M}\triangleq \{v_1, \ldots, v_M\}$ of $M$ RCs (i.e., clients registered at the FL platform) and a set of ${\cal N}\triangleq \{v_1, \ldots, v_N\}$ of $N$ UnRCs.
The set of all RCs and UnRCs is denoted by ${\cal V}\triangleq \{v_1, v_2, \ldots, v_{|{\cal V}|}\}$ and $|{\cal V}|=M+N.$
The FL server first broadcast the information of a FL job to all RCs, but it does not have access to the information of UnRCs.
Moreover, the trust among clients can be established through SIoTs such as WeChat, Twitter, and Facebook, etc., so the trust remains unchanged for a long time.
In addition, without loss of generality, we assume that for each UnRC $v_n\in{\cal N}$, there exists at least one RC ($v_m\in{\cal M}$) having a social relationship with the UnRC.
This assumption is reasonable because if there is no trust value from any RC to UnRC $v_n\in{\cal N}$, then $v_n$ will not be recommended by RCs (i.e., it will not participate in any task).
In this context, we first remove this type of UnRCs when implementing LRef-FedCS.
Denote by ${\boldsymbol W}=[w_{m,n}]\in{\mathbb R}^{M\times N}$ the trust matrix with entries $w_{m,n}\in[0, 1]$ representing the trust value between RC $v_m$ and UnRC $v_n$, where $w_{m,n}=0$ means $v_m$ and $v_n$ are strangers and $w_{m,n}\in(0, 1]$ otherwise.
The key notation definitions are summarized in Table \ref{table:1}.
\begin{table}[!t]
\centering
\caption{List of Notations}
\begin{tabular}{|c|p{2.55cm}|c|p{2.55cm}|}
  \hline
  ${\cal M}$ & set of RCs & ${\cal N}$ & set of UnRCs\\
  \hline
  $\theta[t]$ & local accuracy   & ${\boldsymbol W}$ & matrix of trust values\\
  \hline
  ${\cal N}_m$& \multicolumn{3}{p{6.5cm}|}{extend set of RC $v_m$, ${\cal N}_m={\cal N}\cup\{v_m\}$}\\
  \hline
  ${\cal N}_m^{\text{tru}}[t]$& \multicolumn{3}{p{6.5cm}|}{set of UnRCs trusted by RC $v_m$}\\
  \hline
  \multirow{2}{*}{${\cal N}_{\text{act},m}^{\text{tru}}[t]$ }&\multicolumn{3}{p{6.5cm}|}{  set of UnRCs with personal QoS requirements trusted by RC $v_m$ in time slot $t$  }\\
  \hline
 \multirow{2}{*} {${\cal N}_{m}^{\text{nei}}[t]$} &\multicolumn{3}{p{6.5cm}|}{ set of UnRCs located in the sensing area of RC $v_m$ in time slot $t$}\\
  \hline
\multirow{2}{*}{${\cal N}_{\text{act},m}^{\text{nei}}[t]$}&\multicolumn{3}{p{6.5cm}|}{set of UnRCs with personal QoS requirements located in the sensing area of RC $v_m$ in time slot $t$ }\\
  \hline
  $\phi_m[t]$ & \multicolumn{3}{p{6.5cm}|}{PMDM indicator of RC $v_m$ in time slot $t$ }\\
  \hline
$\alpha_{m,i}[t]$ &\multicolumn{3}{p{6.5cm}|} {LRef-FedCS indicator of $v_i\in{\cal N}_m, v_m\in{\cal M}$}\\
  \hline
\multirow{2}{*}{${\pmb\phi}[t]$} &\multicolumn{3}{p{6.5cm}|}{ participation mode decision of all RCs in time slot $t$, \quad${\pmb\phi}[t]=\{\phi_m[t]\}_{v_m\in{\cal M}}$}\\
\hline
\multirow{2}{*}{${\pmb\alpha}[t]$} &\multicolumn{3}{p{6.5cm}|}{LRef-FedCS policies of all UnRCs in time slot $t$,  ${\pmb\alpha}[t]=\{\alpha_{m,i}[t], v_i\in{\cal N}_m, v_m\in{\cal M}\}$}\\
\hline
${\boldsymbol a}[t]$ &\multicolumn{3}{p{6.5cm}|}{action policy of RCs, ${\boldsymbol a}[t]=\{{\pmb\phi}[t], {\pmb\alpha}[t]\}$}\\
\hline
${\boldsymbol A}_m[t]$ &\multicolumn{3}{p{6.5cm}|}{set of feasible actions of an RC $v_m$ in time slot $t$ }\\
\hline
${\boldsymbol A}[t]$ &\multicolumn{3}{p{6.5cm}|}{set of all feasible actions of all RCs in time slot $t$}\\
\hline
{$\widehat{R}_{m,i}[t]$} &\multicolumn{3}{p{6.5cm}|}{achieved data rate of RC $v_m$ for DParM in time slot $t$}\\
\hline
\multirow{2}{*}{$\widetilde{R}_{m,i}[t]$} &\multicolumn{3}{p{6.5cm}|}{achieved data rate at UnRC $v_i\in{\cal N}_m$ for LRefM in time slot $t$}\\
\hline
\multirow{2}{*}{$\text{T}_{m,i}^{\text{com}}[t]$} & \multicolumn{3}{p{6.5cm}|}{transmission time at client $v_i\in{\cal N}_m, v_m\in{\cal M}$ in time slot $t$ }\\
\hline
\multirow{2}{*}{$\text{T}_{m,i}^{\text{cmp}}[t]$}&\multicolumn{3}{p{6.5cm}|}{computation time at client $v_i\in{\cal N}_m, v_m\in{\cal M}$ in time slot $t$}\\
\hline
\multirow{2}{*}{$\text{E}_{m,i}^{\text{cmp}}[t]$} & \multicolumn{3}{p{6.5cm}|}{energy cost in computation phase at client $v_i\in{\cal N}_m, v_m\in{\cal M}$ in time slot $t$ }\\
\hline
\multirow{2}{*}{$\text{E}_{m,i}^{\text{com}}[t]$} & \multicolumn{3}{p{6.5cm}|}{energy cost in transmission phase at client $v_i\in{\cal N}_m, v_m\in{\cal M}$ in time slot $t$ }\\
\hline
$\tau[t]$ &\multicolumn{3}{p{6.5cm}|}{renewal time point in training round $t$}\\
\hline
\multirow{2}{*}{${\cal MG}$} &\multicolumn{3}{p{6.5cm}|}{matching game between RCs and UnRCs, ${\cal MG}\equiv<{\cal M}, {\cal N}, \succ_{\cal M}, \succ_{\cal N}>$}\\
\hline
${\cal L}_{m}^{\text{rc}}[t]$ & \multicolumn{3}{p{6.5cm}|}{preference list of an RC $v_m\in{\cal M}$}\\
\hline
${\cal L}_{n}^{\text{unrc}}[t]$ & \multicolumn{3}{p{6.5cm}|}{preference list of an UnRC $v_n\in{\cal N}$}\\
\hline
\multirow{2}{*}{${\cal M}_n[t]$} &\multicolumn{3}{p{6.5cm}|}{set of RCs that can collect the information of UnRC $v_n$ in time slot $t$ }\\
\hline
\end{tabular}
\label{table:1}
\end{table}
\section{Problem Formulation}
\subsection{Participation Mode Decision Model }
Before introducing the LRef-FedCS model, we first define the participation mode decision model (PMDM) shown in Fig. \ref{fig:1}.
In the PMDM, each RC acts according to its current personal schedule (sufficient or insufficient computing resources), where an action here represents an RC's decision to directly execute a training task or recommend a trustworthy UnRC as a candidate, these two actions are termed  DParM and LRefM, respectively.
We assume that each RC acts rationally,  with actions predicated on its current state and the trust value with the candidate UnRCs.
A binary variable $\phi_m[t] $ is defined to represent the two participation modes for RC $v_m\in{\cal M}$ in time slot $t$ (i.e., corresponding to the $t\text{-th}$ training task announced by the FL server), shown as follows:
\begin{equation}\label{s:3}
\phi_m[t]=\begin{cases}
0, &\hspace{-1em}\text{~if~} v_m \text{~uses DParM on the training task~}t,\\
1, &\hspace{-1em}\text{~if~} v_m \text{~uses LRefM on the training task~}t.
\end{cases}
\end{equation}
The introduction of PMDM is conducive to describe the LRef-FedCS process.
We further classify the LRefM into two modes based on the current requirements of the recommended UnRC:
{1) Partial-LRefM:} the mode is for recommending UnRCs with their own quality of service (QoS) requirements.
That is, the recommended UnRCs need to reserve certain resources to guarantee their personal QoS requirements in some training rounds;
{2) Full-LRefM:} the mode is proposed for the scenario  where the recommended UnRCs are inactive, i.e., they do not have specific QoS requirements in some training rounds.

\subsection{ Communication Phase }
It is evident that the LRef-FedCS strategy is inextricably linked to the RCs' participation mode decision-making process.
In order to facilitate the learner referral process of RCs, we denote an auxiliary notation in the view of each RC $v_m\in{\cal M}$ by ${\cal N}_m\triangleq {\cal N}\cup\{v_m\}.$
Therefore, the achieved data rate of client $v_i\in{\cal N}_m$ recommended by RC $v_m \in{\cal M}$ in time slot $t$ can be written as
\begin{align}\label{l:1}
R_{m,i}[t]=\phi_m[t]\alpha_{m,i}[t]\widehat{R}_{m,i}[t]
+(1-\phi_m[t])\alpha_{m,i}[t]\widetilde{R}_{m,i}[t],
\end{align}
where we set $\alpha_{m,i}[t]=1$ if $v_i\in{\cal N}_m$ is successfully recommended by RC $v_m$, otherwise $\alpha_{m,i}[t]=0$.
Moreover, $\widehat{R}_{m,i}[t]$ and $\widetilde{R}_{m,i}[t]$ represent the achieved data rate at client $v_i\in{\cal N}_m$ associated with RC $v_m$ under the DParM and LRefM, respectively.
It is noted that executing the DParM means that the recommended client is RC $v_m$ itself.
That is, $\alpha_{m,m}=1$ iff $\phi_m[t]=1$.
For ease of exposition, we use subscript $i$ to express the achieved data rate for the DParM.
In order to upload and update trained parameters, we consider an OFDMA protocol to establish wireless communications for clients.
Thus, for the DParM,  the achieved data rate of $v_i\in{\cal N}_m$ associated with RC $v_m$ in time slot $t$ can be expressed as
\begin{equation}\label{l:2}
\begin{aligned}
\widehat{R}_{m,i}[t]=B\log\left(1+\frac{h_{i,0}[t]p_i^{\max}}{N_0B} \right),
\end{aligned}
\end{equation}
where $N_0$ and $p_i^{\max}$ are the spectral density of the white Gaussian noise and transmit power budget of $v_i$, respectively,  $h_{i,0}[t]$ denotes the channel gain between $v_i$ and the FL server, and $B$ represents the transmission bandwidth.

Next, we derive  the achieved data rate  for LRefM, according to the states of the recommended UnRCs.
Denote by  ${\cal N}_m^{\text{tru}}[t]$ the set of UnRCs trusted by RC $v_m\in{\cal M}$.
Specifically, ${\cal N}_m^{\text{tru}}[t]:=\{v_n:w_{m,n}>0, v_n\in{\cal N}\}.$
In time slot $t$, the set of UnRCs that have QoS requirements in ${\cal N}_m^{\text{tru}}[t]$ is further denoted by ${\cal N}_{\text{act},m}^{\text{tru}}[t]$.
${\cal N}_{\text{act},m}^{\text{tru}}[t]$ can be established by a certain neighbor discovery algorithm and matched randomly in D2D communications \cite{gao2019dynamic}, and the maximum distance of client-to-client (C2C) is $d.$\footnote{To exploit individual disparities of RCs and UnRCs in terms of the surplus of communication- and computation-resource,
we simply adapt the data rate of C2C link as the QoS metric for UnRCs. In addition, we set the distance of C2C link is $5\text{m}$ \cite{gao2021reflection}.}
As mentioned in \cite{lin2021friend}, there is sufficient motivation among friends to contribute to collaborative learning.
Thus, the trust value between RCs and UnRCs can reflect the intensity of resource sharing to a certain extent.
Consequently, given the above sets as well as the Partial- and Full-LRefM definitions, the trust-aided resource scheduler  $\Xi(\cdot)$ for RCs, which represents a mapping from set ${\boldsymbol W}=\{w_{m,n}\}$ to the scheduled resource, which is formulated as $\Xi: {\boldsymbol W}\mapsto {\mathbb R}$, can be mathematically expressed as follows
\begin{align}\label{b:1}
\Xi(w_{m,i})\hspace{-0.3em}=\hspace{-0.4em}\begin{cases}
1-w_{m,i}, &\hspace{-0.5em}\text{if~}v_i\in{\cal N}_{\text{act},m}^{\text{tru}}[t],\\
1, &\hspace{-0.5em}\text{if~}v_i\in{\cal N}_m^{\text{tru}}[t]\setminus {\cal N}_{\text{act},m}^{\text{tru}}[t].
\end{cases}
\end{align}
Given the trust-aided resource scheduler  $\Xi(\cdot)$ for RCs, the data rate of the recommended $v_i\in{\cal N}_m$ can be shown as
\begin{align}
\widetilde{R}_{m,i}[t]=&\Xi(w_{m,i})B
\log\left(1+\frac{h_{i,0}[t]\Pi(w_{m,i})p_i^{\max}}
{N_0\Xi(w_{m,i})B} \right),\label{l:3}
\end{align}
where $\Pi(\cdot)$ is the trust-aided resource scheduler for the participant client $v_i$ associated with RC $v_m$, which is defined as follows
\begin{align}\label{b:2}
\Pi(w_{m,i})=\begin{cases}
&w_{m,i}, \qquad\qquad\text{if~} v_i\in{\cal N}_{\text{act},m}^{\text{tru}}[t],\\
&\frac{w_{m,i}}{\sum_{v_i\in{\cal N}_m^{\text{tru}}[t]}w_{m,i}},\qquad\qquad\qquad\\
&\qquad ~~~~ \text{if~} v_i\in{\cal N}_m^{\text{tru}}[t]\setminus{\cal N}_{\text{act},m}^{\text{tru}}[t].
\end{cases}
\end{align}

Let the data size of the local model parameters in learner $v_i\in{\cal N}_m$ to be updated be $C_i$.
So the time of uploading the parameters  to the server is given by
\begin{equation}\label{l:4}
\text{T}_{m,i}^{\text{com}}[t]={C_i}/{R_{m,i}[t]}.
\end{equation}
Correspondingly, the energy consumption for uploading the training parameters of $v_i\in{\cal V}$ is
\begin{align}
\text{E}_{m,i}^{\text{com}}[t]=&\left\{\phi_m[t]\alpha_{m,i}[t]
+(1-\phi_m[t])\alpha_{m,i}[t]\right.\notag\\
&\left.\times\Pi(w_{m,i})\right\}p_i^{\max}\text{T}_{m,i}^{\text{com}}[t].\label{l:5}
\end{align}

\subsection{Local Computation Phase}
For each client $v_i\in{\cal N}_m$, the budget of the CPU cycle frequency (cycle/byte) is represented by $f_i$.
Denote by $\rho_i f_i^{\zeta}$ the computation power of $v_i\in{\cal N}_m$, where $\rho_i$ is a constant that depends on the average switched capacitance and the average activity factor, and $\zeta>2$ is a constant.
Similarly, upon introducing the resource scheduler $\Pi(\cdot)$ on UnRC side,  the computation time at  $v_i\in{\cal N}_m$ for one local iteration can be expressed as
\begin{align}\label{l:6}
\text{T}_{m,i}^{\text{cmp}}[t]=
\frac{Q_i[t]b_i}{\left\{\phi_m[t]\alpha_{m,i}[t]
+(1-\phi_m[t])\alpha_{m,i}\Pi(w_{m,i})\right\}f_i},
\end{align}
where $Q_i$ and $b_i$ represent the numbers of training samples and CPU cycles that are required to process a data sample, respectively.

Accordingly, the associated  energy consumption for computation at $v_i\in{\cal N}_m$ for one local iteration is given by
\begin{align}
\text{E}_{m,i}^{\text{cmp}}[t]=&\rho_iQ_i[t]b_i
\left\{\phi_m[t]\alpha_{m,i}[t]+(1-\phi_m[t])\right.\notag\\
&\left.\times\alpha_{m,i}[t]\Pi(w_{m,i})\right\}^{\zeta-1}f_i^{\zeta-1}.\label{l:7}
\end{align}
Using \eqref{l:4}-\eqref{l:7},  the time cost and energy consumption for each $v_i\in{\cal N}_m$ in one global epoch are shown to be
\begin{align}
\text{T}_{m,i}[t]&=\log(1/\theta[t])\text{T}_{m,i}^{\text{cmp}}[t]+\text{T}_{m,i}^{\text{com}}[t],\\
\text{E}_{m,i}[t]&=\log(1/\theta[t])\text{E}_{m,i}^{\text{cmp}}[t]+\text{E}_{m,i}^{\text{com}}[t].
\end{align}

Additionally, in the HieIoT network FL system, we assume that two UnRCs can be randomly matched to form a C2C link within a distance of $d$,  each serving as a transceiver for the other.
In time slot $t$, if  UnRC $v_i\in{\cal N}_{\text{act},m}^{\text{tru}}[t]\cap{\cal N}_m$ is recommended as a learner, the instantaneous data rate of the C2C pair link $(v_i,v_j)$ can be expressed as
\begin{align}
R_{i,j}^{(m)}[t]= \alpha_{m,i}[t]w_{m,i}B\log\left(1+
\frac{g_{i,j}[t](1-w_{m,i})p_i^{\max}}
{N_0w_{m,i}B}\right),\label{add:2}
\end{align}
where $g_{i,j}[t]$ represents the channel  gain of  link $(v_i,v_j)$ in time slot $t$.
The corresponding transmit power and bandwidth allocated for the $(v_i, v_j)$ link are $(1-w_{m,i})p_i^{\max}$ and $w_{m,i}B$, respectively. In principle, each UnRC can randomly match any other UnRC within its selection range.

\subsection{Problem Formulation}

{\bf\em Decision variables.}
As formulated in Section V.A, the participation mode decision-making and LRef-FedCS process as they align with the timeline of announced training tasks (or time slots),  can be regarded as a renewal system, which is shown in Fig. \ref{fig:2}.
The set of durations for successive training tasks is $\{\text{T}[1], \text{T}[2], \ldots\}$.
Define $\tau[1]=0$, and for each training task $t$ define $\tau[t]$ as the $t\text{th}$ {\em renewal time instance}, i.e.,  $\tau[t]:\triangleq \sum\nolimits_{l=1}^{t} \text{T}[l].$
Note that  the defined {\em renewal time instance} is the time point at which the FL server announces the task.
As indicated before, the status\footnote{Status of RCs is determined according to whether its computation resources are surplus.} and classification of clients are critical for  performing the PMDM and LRef-FedCS.
As mentioned in Section V.A,  we give a more precise definition of binary variable $\alpha_{m,i}[t]$ on the extended set ${\cal N}_m$ of each RC $v_m\in{\cal M}$, where $\alpha_{m,m}[t]=1$ iff $\phi_m[t]=1.$
Moreover, it is assumed that each RC can only recommend at most one UnRC and each UnRC can only be
associated with one RC in any time slot, i.e., $\sum_{i\in{\cal N}_m}\alpha_{m,i}[t]\leq 1, \forall v_m\in{\cal M}$ and $\sum_{m=1}^M\alpha_{m,i}[t]\leq 1, \forall v_i\in{{\cal N}_m}.$
Let ${\boldsymbol A}[t]=\{\boldsymbol a[t]\}$ be the set of RCs' actions at the {\em renewal time instance} $\tau[t]$ with entries $\boldsymbol a[t]$ composed of PMDM ${\pmb\phi[t]}$ and LRef-FedCS ${\pmb\alpha}[t]$, respectively.
That is, ${\boldsymbol a}[t]\triangleq \{{\pmb\phi}[t], {\pmb\alpha}[t]\},$ where ${\pmb\phi}[t]=\{\phi_m[t]\in\{0, 1\}, \forall v_m\in{\cal M}\}$ and ${\pmb\alpha}[t]=\{\alpha_{m,i}[t]\in\{0, 1\}, v_i\in{\cal N}_m, \forall v_m\in{\cal M}\}.$
Therefore, at each {\em renewal time instance}, the FL server makes a decision on $\boldsymbol a[t].$
\begin{figure}[!t]
	\centering
	\includegraphics[width=3.45in]{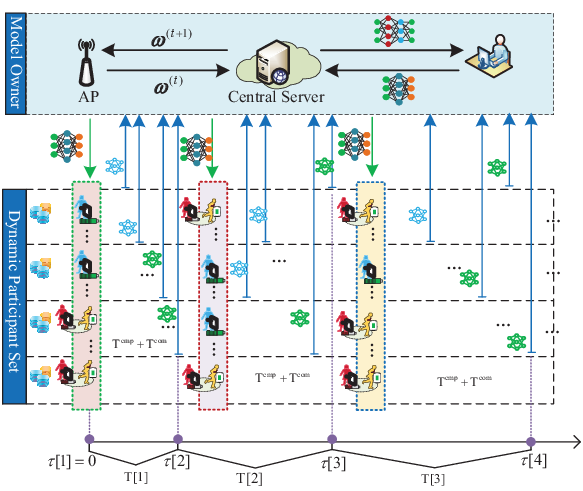}
	\caption{Timeline illustration of the renewal process for federated learning over the hierarchical IoT network.}
	\label{fig:2}
\end{figure}

{\bf\em Constraints and assumptions.} {\em (1) Long-term fairness requirement of FL training:}
The cost-effective action decision-making may not  be the best strategy in terms of the learning quality, and the requirement of FL training in action decision-making is another factor that needs to be taken into account.
Moreover,  the FL server can only broadcast the  training task to RCs.
In this context, how to eliminate the  ``{\em Tragedy of the Commons}'' \cite{milinski2002reputation} to sustain a long-term learning quality is another factor that should be considered.
Accordingly,  we introduce a time-average constraint (\ref{add:11}),  which ensures that the expected proportion of recommended learners for each RC $v_m\in{\cal M}$ does not lower than $\Delta\footnote{Typically, the value of $\Delta$ is less than or equal to $\frac{M}{M+N}$. In this paper, we set $\Delta=\frac{M}{M+N}$}.$
\begin{align}
\varlimsup_{R\rightarrow\infty} \frac{1}{R}
\sum\nolimits_{t=1}^{R}{\mathbb E}\left\{\sum\nolimits_{v_i\in{\cal N}_m} \alpha_{m,i}[t]\right\}\geq \Delta, \forall v_m\in{\cal M}.\label{add:11}
\end{align}

{\em (2) Duration of local iteration:}
To alleviate the straggler issue, for each potential client $v_i\in{\cal N}_m, v_m\in{\cal M}$, the computing time of one local iteration is upper bounded by $\text{T}_{\max}^{\text{cmp}}$, i.e.,
\begin{align}
&\frac{Q_i[t]b_i}{\sum\nolimits_{v_m\in{\cal M}}[\phi_m[t]\alpha_{m,i}[t]
+(1-\phi_m[t])\alpha_{m,i}[t]\Pi(w_{m,i})]f_i}\notag\\
&\qquad \qquad\qquad\qquad\qquad\leq \text{T}_{\max}^{\text{cmp}}, \forall v_i\in{\cal N}_m, v_m\in{\cal M}. \label{add:12}
\end{align}

{\em (3) Long-term QoS requirement of UnRCs:}
As stated in Section V.A, trust-aided resource scheduling can incentivize UnRCs to temporarily replace some trustworthy RCs to participate in the training process.
To model such QoS requirements, it is assumed that UnRCs' best interest is to pursue long-term requirements, which can be described as
\begin{align}
&\varlimsup\limits_{R\rightarrow \infty} \frac{\sum_{t=1}^{R}\sum_{m=1}^M\alpha_{m,n}[t]
{\cal I}_{{\cal N}_{\text{act},m}^{\text{tru}}[t]}(n)
{\mathbb E}\left\{ R_{n}^{\text{c2c}}[t]\right\}}
{\sum_{t=1}^{R}\sum_{m=1}^M\alpha_{m,n}[t]{\cal I}_{{\cal N}_{\text{act},m}^{\text{tru}}[t]}(n)}\geq R_{n,\min}^{\text{c2c}},\notag\\
&\quad \quad \quad \quad \forall v_n\in{\cal N}, \label{add:13}
\end{align}
where $R_{n,\min}^{\text{c2c}}$ is defined as the minimum possible data rate that $v_n\in{\cal N}$ can experience on its C2C link, $R_n^{\text{c2c}}[t]$ represents the rate of the C2C link that involves  client $v_n\in{\cal N}$, and ${\cal I}_{{\cal N}_{\text{act},m}^{\text{tru}}[t]}(\cdot)$ is an introduced indicator function on subset ${\cal N}_{\text{act},m}^{\text{tru}}[t]$,  mathematically expressed as:
\begin{equation}\label{add:15}
\begin{aligned}
{\cal I}_{{\cal N}_{\text{act},m}^{\text{tru}}[t]}(n)=
\begin{cases}
1, &\text{if~} v_n \in {\cal N}_{\text{act},m}^{\text{tru}}[t],\\
0, &\text{otherwise}.
\end{cases}
\end{aligned}
\end{equation}
Through defining ${\cal I}_{{\cal N}_{\text{act},m}^{\text{tru}}[t]}(n)$, the denominator of \eqref{add:13} can be used to calculate participation rounds of $v_n\in{\cal N}$ with the C2C links.

We employ the weighted sum method to address  the trade-off among multiple objectives using parameters $\lambda^{\text{t}}$ and $\lambda^{\text{e}}$ with $0\leq \lambda^{\text{t}}, \lambda^{\text{e}}\leq 1, \lambda^{\text{t}}+\lambda^{\text{e}}=1$.
The WSET at $v_i\in{\cal N}_m, \forall v_m\in{\cal M}$ for training task $t$ is given by
\begin{equation}\label{add:16}
\begin{aligned}
{\cal G}_{m,i}[t]=\frac{1}{1-\theta[t]}
\left(\lambda^{\text{t}}\text{T}_{m,i}[t]
+\lambda^{\text{e}}\text{E}_{m,i}[t]\right).
\end{aligned}
\end{equation}
It is worth noting that our model is sufficiently broad, since weights can be adjusted to emphasize a sub-objective over another, potentially honing the focus on minimizing solely energy consumption or time cost.
Given the above discussions, we introduce the following action policy making problem
\begin{align}
&\min_{\{\theta[t], {\boldsymbol a}[t]\in{\boldsymbol A[t]}\}_{t=1}^{R}}\varlimsup_{{R\rightarrow\infty}}\frac{1}{R}
\sum\nolimits_{t=1}^{R}\max_{v_i\in{\cal N}_m, v_m\in{\cal M}}{\cal G}_{m,i}[t]\label{add:17}\\
&\text{s.t.~}(\ref{add:11}), (\ref{add:12}), (\ref{add:13}),\tag{\ref{add:17}a}\\
&\hspace{1.5em}0\leq \theta[t]\leq 1, \tag{\ref{add:17}b}\\
&\hspace{1.5em}\phi_m[t]\in\{0, 1\}, \forall v_m\in{\cal M}, \tag{\ref{add:17}c}\\
&\hspace{1.5em}\alpha_{m,i}[t]\in\{0, 1\}, \forall v_i\in{\cal N}_m, v_m\in{\cal M},\tag{\ref{add:17}d}\\
&\hspace{1.5em}\sum\nolimits_{m=1}^{M}\alpha_{m,i}[t]\leq 1, \forall v_i\in{\cal N}_m,\tag{\ref{add:17}e}\\
&\hspace{1.5em}\sum\nolimits_{v_i\in{\cal N}_m} \alpha_{m,i}[t]\leq 1, \forall v_m\in{\cal M},\tag{\ref{add:17}f}
\end{align}
where ${\boldsymbol a}[t]$ captures the actions of participation mode decision and LRef-FedCS at {\em renewal time instance} $\tau[t]$, which is our optimization target.
Intuitively, our goal is to minimize the long-term WSET of the worst-case participant subject to the long-term requirements of FL training (\ref{add:11}) and UnRCs (\ref{add:13}), which tolerate short-term violation, along with the non-negotiable additional hard constraint, as denoted in (\ref{add:12})
It is noted taht, (\ref{add:13}) is an additional constraint indicating that UnRCs will act as temporary participants for some trustworthy RCs only when the long-term QoS exceeds a certain level.

{\bf\em Challenges:}
One would note that problem (\ref{add:17}) is a time-coupling scheduling problem without constraints (\ref{add:12}) and (\ref{add:17}b)-(\ref{add:17}f), regarding the long-term objective and requirements of FL training (\ref{add:11}) and UnRCs (\ref{add:13}).
Our primary challenges are mainly derived from the time-coupling constraints (\ref{add:11}) and (\ref{add:13}), which are quite difficult for an offline solution to deal with.
In addition, suppose that the participation mode decision and the LRef-FedCS process are made before the real training process.
Therefore, for an alternative sub-optimal solution, in the following section, we will elaborate on our transformation of the offline problem to an online scheduling problem via Lyapunov-concept optimization so as to satisfy the time-coupling requirements of FL training and UnRCs.

\section{Online Algorithm Design}
\subsection{Problem Transformation under Lyapunov Framework}
In this subsection, we first utilize the Lyapunov optimization framework to transform (\ref{add:17}) into an online problem.
First, we introduce a virtual queue $\gamma_m[t]$  for $v_m\in{\cal M}$ to deal with the long-term fairness constraint (\ref{add:11}).
Likewise, we denote by $z_n[t]$ the virtual queue for $v_n\in{\cal N}$, to deal with the long-term QoS requirements  (\ref{add:13}).
Specially, $\gamma_m$ and $z_n$ evolve in the FL process through complying with the following rules:
\begin{align}
\gamma_m[t+1]=&\left[ \gamma_m[t]+\Delta-\sum\nolimits_{v_i\in{\cal N}_m}\alpha_{m,i}[t]
\right]^{+}, \forall v_m\in{\cal M},\label{add:18}\\
z_n[t+1]=&\left[z_n[t]+\sum\nolimits_{v_m\in{\cal M}}\alpha_{m,n}[t]{\cal I}_{{\cal N}_{\text{act},m}^{\text{tru}}[t]]}(n)\right.\notag\\
&\left.\times\left(R_{n,\min}^{\text{c2c}}-R_n^{\text{c2c}}[t]\right)
\right]^{+},\forall v_n\in{\cal N},\label{add:18-1}
\end{align}
where $[x]^{+}$ is equivalent to $\max\{x, 0\}.$

For the sake of briefness, we eliminate all time coupling constraints and transform (\ref{add:17}) to the following problem
\begin{align}
\min_{\theta[t],{\boldsymbol a}[t]}&
\left\{\max_{v_i\in{\cal N}_m,v_m\in{\cal M} }V{\cal G}_i[t]
+\sum\nolimits_{v_m\in{\cal M}}\gamma_m[t]\right.\notag\\
&\left.\left(\Delta-\sum\nolimits_{v_i\in{\cal N}_m} \alpha_{m,i}[t]\right)+
\sum\nolimits_{v_n\in{\cal N}}z_n[t]\right.\notag\\
&\left.\left(\sum\nolimits_{v_m\in{\cal M}}\alpha_{m,n}[t]{\cal I}_{{\cal N}_{\text{act},m}^{\text{tru}}[t]}(n)
\left(R_{n,{\min}}^{\text{c2c}}-R_n^{\text{c2c}}[t]\right)\right)
\right\}  \label{add:19}\\
\text{s.t.~}& (\ref{add:12}) \text{~and~} (\ref{add:17}\text{b})-(\ref{add:17}\text{f}). \tag{\ref{add:19}a}
\end{align}
It can be seen that objective (\ref{add:19}) is non-convex and the constraints (\ref{add:17}c)-(\ref{add:17}d) imply that (\ref{add:19}) is an integer optimization problem.
Therefore,  problem (\ref{add:19}) is a mixed-integer non-convex program and there are no efficient methods to solve such kind of problems.
Employing the alternating optimization techniques \cite{csiszar1984information}, we separately and iteratively solve ${\boldsymbol a}[t]$ and local model accuracy $\theta[t]$.
In particular, we first solve for the optimal action decision $\boldsymbol a[t]$ given a fixed $\theta[t]$, and then derive the optimal local model accuracy when $\boldsymbol a[t]$ is fixed.
In time slot $t$, for a given local accuracy $\theta[t]$,  due to the unique LRec-FedCS constraints (\ref{add:17}d)-(\ref{add:17}e), the number of possible LRec-FedCS policies is reduced from $2^{M(K+1)}$ to $(K+1)^M$.
However, the number of participation modes is $2^M$.
Hence, the possible action space of problem (\ref{add:19}) is $(K+1)^M$.
Solving (\ref{add:19}) requires an exhaustive combinatorial search over all these possible action policies, where an NP-hard problem formulated by objective (\ref{add:19}) and constraint (\ref{add:17}b) must be solved for each of these actions.
This motivates the pursuit of computationally efficient, near-optimal solutions.
Before proposing a distributed method in the ensuing section, we first detail the procedure for obtaining the action space ${\boldsymbol A}[t]$, which is then used as a comparison method.

\subsection{Centralized Method for LRef-FedCS}
Upon introducing the notation of ${\cal N}_m$ for each $v_m\in{\cal M}$ and the extend definition of $\alpha_{m,i}[t], v_i\in{\cal N}_m$, PMDM is implicitly determined by the LRef-FedCS procedure, in the following,  so that our main concern focus on the latter, i.e., ${\pmb\alpha}[t].$
To visualize the interaction between the participation mode decision and the LRef-FedCS on the feasible solution space of problem (\ref{add:19}),  a generic function is defined on set ${\cal N}_m^{\text{tru}}[t]\cup\{m\}$ of $v_m$ to assign its computation time of one local iteration relative to the selected client, which is mathematically expressed as
\begin{equation}\label{a:1}
\text{T}_{m,i}^{\text{cmp}}[t]=\begin{cases}
\check{\text{T}}_{m,i}^{\text{cmp}}[t], &\text{if~}v_i=v_m,\\
\hat{\text{T}}_{m,i}^{\text{cmp}}[t], &\text{if~}v_i\in {\cal N}_{\text{act}, m}^{\text{tru}}[t],\\
\tilde{\text{T}}_{m,i}^{\text{cmp}}[t], &\text{if~}v_i\in{\cal N}_m^{\text{tru}}[t]\setminus{\cal N}_{\text{act}, m}^{\text{tru}}[t].
\end{cases}
\end{equation}
Given the above definitions, we design a centralized method to obtain the feasible action space of problem (\ref{add:19}) according to the structure of constraint (\ref{add:12}).
The detail process is shown as follows.

{\bf\em Step 1}. For each $v_m\in{\cal M}$, letting $\alpha_{m,i}[t]=1, v_i\in{\cal N}_m$ at {\em renewal time instance} $\tau[t]$ (i.e., corresponding to the $t\text{-th}$ announced training task) and using \eqref{a:1}, then constraint (\ref{add:12}) can be rewritten as
    \begin{align}
    \phi_m[t]\check{\text{T}}_{m,i}^{\text{cmp}}[t]&+(1-\phi_m[t])
    \{\xi_m[t]\hat{\text{T}}_{m,i}^{\text{cmp}}[t]\notag\\
    &+(1-\xi_m[t])\tilde{\text{T}}_{m,i}^{\text{cmp}}[t]\}
    \leq \text{T}_{\max}^{\text{cmp}}, \label{add:28}
    \end{align}
where $\xi_m[t]\in\{0, 1\}$ illustrates the participation mode decision of $v_m\in{\cal M}$, i.e., if $\xi_m[t]=1$, $v_m$ will select Partial-LRecM, otherwise, select Full-LRecM.

 {\bf\em Step 2}. At {\em renewal time instance} $\tau[t]$, the {\em Feasible Action Set Identification} of each $v_m\in{\cal M}$ will become a feasibility test for each client $v_i\in{\cal N}_m$.
In our scenario, the centralized method starts by the FL server, collecting global system information at each $\tau[t]$.
The global information of the entire system consists of the channel state information, social relationships among RCs and UnRCs, the training data size of each client, and the schedule of RCs, denoted by  ${\boldsymbol S}[t]=\{ {\boldsymbol H}[t], {\boldsymbol W}[t], {\boldsymbol Q}[t], {\boldsymbol G}[t]\}$.
Specifically, when considering $v_i\in{\cal N}_m, v_m\in{\cal M}$, its local data arrival process obeys a Possion distribution, i.e.,  ${\boldsymbol Q}[t]\triangleq\{Q_i[t]: Q_i[t]\sim {\cal P}(\lambda_i), v_i\in{\cal N}_m, v_m\in{\cal M}\}$.
The status information of RCs is described by ${\boldsymbol G}[t]\triangleq \{G_m[t]: G[t]\sim \text{Ber}(\varepsilon), v_m\in{\cal M}\}.$

{\bf\em Step 3.}  After collecting the global information of system ${\boldsymbol S}[t]$,  the set of RCs with sufficient computing resources will be obtained, represented by
${\cal M}_{\text{idle}}[t]$.
Upon introducing the PMDM, we further assume that constraint (\ref{add:28}) is always satisfied for $v_m\in{\cal M}_{\text{idle}}[t]$, i.e., $\phi_m[t]\check{\text{T}}_m^{\text{cmp}}[t]\leq \text{T}_{\max}^{\text{cmp}}$, which means  $\alpha_{m,m}[t]=1$ is the action of RC $v_m$, whilst it is not an element of the possible action set ${\boldsymbol A}_{m}[t]=\{\phi_m[t], \alpha_{m,i}[t]\}_{v_i\in{\cal N}_m}$ of $v_m\in{\cal M}\setminus{\cal M}_{\text{idle}}[t]$.
For each $v_m\in{\cal M}\setminus{\cal M}_{\text{idle}}[t]$, fix $\xi_m[t]=1$, we let $\alpha_{m,i}[t]=1$ for each $v_i\in{\cal N}_m$ if $\hat{\text{T}}_i^{\text{cmp}}[t]\leq \text{T}_{\max}^{\text{cmp}}$, otherwise, $\alpha_{m,i}[t]=0$.
Likewise, given $\xi_m[t]=0$, $\alpha_{m,i}[t]$ will set to be $1$ for each $v_i\in{\cal N}_m$ if $\tilde{\text{T}}_i^{\text{cmp}}[t]\leq \text{T}_{\max}^{\text{cmp}}$, or $\alpha_{m,i}[t]= 0$, otherwise.
Subsequently, at {\em renewal time instance} $\tau[t]$, each $v_m\in{\cal M}$ updates its possible action set ${\boldsymbol A}_{m}[t]$ and sends the result back  to the FL server, and then the action space ${\boldsymbol A}[t]$ is the combination of all $\{{\boldsymbol A}_{m}[t]\}_{v_m\in{\cal M}}$.

{\bf\em Step 4}.  In order to find the optimal action policy for (\ref{add:19}) satisfying constraints (\ref{add:17}e) and (\ref{add:17}f), the possible action policies deduced by the above steps can be expressed in matrix format.
Specifically,  let $\pmb\Lambda[t]=\{\blacktriangle_{mi}[t]\}\in{\mathbb R}^{M\times |{\cal V}|}$ describe the possible solution space with  $\pmb\Lambda_{m,:}[t]\in{\mathbb R}^{1\times |{\cal V}|}$ representing the possible action policies of $v_m\in{\cal M}$, where non-zero elements of $\pmb\Lambda_{m,:}[t]\in{\mathbb R}^{1\times |{\cal V}|}$ form the possible action set ${\boldsymbol A}_m[t]$.
Moreover, define the $M\times M $ matrix $\pmb\Lambda^1$ as the first block of $\pmb\Lambda$, and the second block of $\pmb\Lambda$ is denoted by $\pmb\Lambda^2\in{\mathbb R}^{M\times N},$ so yielding $\pmb\Lambda=[\pmb\Lambda^1 ~ \pmb\Lambda^2].$
For an exhaustive search over ${\boldsymbol A}[t]$, the action of $v_m\in{\cal M}_{\text{idle}}$ is determined  via filtering all non-zero elements of $\pmb\Lambda^1$, and then introduce an auxiliary variable $\pmb\Lambda^2(\pmb\Lambda_{-m,:}[ ])$ which represents the matrix after deleting some rows from ${\pmb\Lambda}^2$ corresponding to the elements in set ${\cal M}_{\text{idle}}$.
Given the auxiliary matrix $\pmb\Lambda^2(\pmb\Lambda_{-m,:}[ ])$, a possible action policy can be deduced from matrix entries that satisfy $\sum_{i=M+1}^{|{\cal V}|}\blacktriangle_{mi}[t]=1$ and $\sum_{v_m\in{\cal M}\setminus {\cal M}_{\text{idle}}}\blacktriangle_{mi}[t]=1$.
Given the sensing range limitation of each $v_m\in{\cal M}$, the number of combinations over the possible action set ${\boldsymbol A}[t]$ is upper bounded by $\Pi_{v_m\in{\cal M}\setminus{\cal M}_{\text{idle}}} \left|{\boldsymbol A}_m\right|$.

\subsection{Local Model Accuracy Optimization}
With the obtained optimal action policy ${\boldsymbol a}[t]$ for the announced training task $t$, the optimal local model accuracy $\theta[t]$ can be deduced as follows
\begin{align}
\min_{\theta[t]\in[0,1]}\max_{v_i\in{\cal N}_m,v_m\in{\cal M}}&V\left\{\frac{\log\left( \frac{1}{\theta[t]}\right)}{1-\theta[t]}A_{m,i}(\boldsymbol a[t])+\frac{B_{m,i}(\boldsymbol a[t])}{1-\theta[t]}\right\}\notag\\
&+C(\boldsymbol a[t]).\label{add:29}
\end{align}
In the following, we will omit symbol $t$, and $A_{m,i}[t], B_{m,i}[t]$, and $C(\boldsymbol a[t])$ for each $v_i\in{\cal N}_m, v_m\in{\cal M}$ are given as follows
\begin{align}
A_{m,i}(\boldsymbol a)=&\phi_m\alpha_{m,i}\check{{G}}_{m,i}^{\text{cmp}}
+(1-\phi_m)[\xi_m\alpha_{m,i}\hat{{G}}_{m,i}^{\text{cmp}}\notag\\
&+(1-\xi_m)\alpha_{m,i}
\tilde{{G}}_{m,i}^{\text{cmp}}],\label{a:3}\\
B_{m,i}(\boldsymbol a)=&\phi_m\alpha_{m,i}\check{{G}}_{m,i}^{\text{com}}
+(1-\phi_m)[\xi_m
\alpha_{m,i}\hat{{G}}_{m,i}^{\text{cmp}}\notag\\
&+[1-\xi_m\alpha_{m,i}
\tilde{{G}}_{m,i}^{\text{cmp}}],\label{a:4}\\
C(\boldsymbol a)=&\sum\nolimits_{v_m\in{\cal M}}\gamma_m\left[\Delta-\sum\nolimits_{v_i\in{\cal N}_m}\alpha_{m,i}\right]+\notag\\
&\sum\nolimits_{v_n\in{\cal N}}z_n[t]
\left[\sum\nolimits_{v_m\in{\cal M}}\alpha_{m,n}{\cal I}_{{\cal N}_{\text{act},m}^{\text{nei}}}(n)\right.\notag\\
&\left.\times(R_{n,\min}^{\text{c2c}}-R_n^{\text{c2c}})\right],\label{a:5}
\end{align}
where $\check{G}_{m,i}^{u}[t]=\lambda^{\text{t}}\check{\text{T}}_{m,i}^{u}[t]+\lambda^{\text{e}}\check{\text{E}}_{m,i}^{u}[t]$,
$\hat{G}_{m,i}^{u}[t]=\lambda^{\text{t}}\hat{\text{T}}_{m,i}^{u}[t]+\lambda^{\text{e}}\hat{\text{E}}_{m,i}^{u}[t]$, and $\tilde{G}_{m,i}^{u}[t]=\lambda^{\text{t}}\tilde{\text{T}}_{m,i}^{u}[t]+\lambda^{\text{e}}\tilde{\text{E}}_{m,i}^{u}[t]$ with $u\in\{\text{cmp}, \text{com}\}.$
Similar to the structure of (\ref{a:1}), for $\bigstar\in\{\text{T}, \text{E}\}$, we have
\begin{equation}\label{a:2}
\bigstar_{m,i}^{\text{com}}[t]=\begin{cases}
\check{\bigstar}_{m,i}^{\text{com}}[t], &\text{if~}v_i=v_m,\\
\hat{\bigstar}_{m,i}^{\text{com}}[t], &\text{if~}v_i\in {\cal N}_{\text{act}, m}^{\text{nei}}[t],\\
\tilde{\bigstar}_{m,i}^{\text{com}}[t], &\text{if~}v_i\in{\cal N}_{m}^{\text{tru}}[t]\setminus{\cal N}_{\text{act}, m}^{\text{tru}}[t].
\end{cases}
\end{equation}
It is noted that, for each $v_i\in{\cal N}_m, v_m\in{\cal M}$, $A_{m,i}(\boldsymbol a[t])$ and $B_{m,i}(\boldsymbol a[t])$ are the weighted sum of the time cost and energy consumption of the computation and communication phases, respectively.
$C(\boldsymbol a[t])$ accounts for the penalty due to the long-term constraints on the action decision ${\boldsymbol a}[t]$ in time slot $t$.
Intuitively, given the action decision ${\boldsymbol a}[t]$,  $A_{m,i}(\boldsymbol a[t])$, $B_{m,i}(\boldsymbol a[t])$, and $C(\boldsymbol a[t])$ become constants.

In each time slot $t$, for each combination of feasible action decisions of $v_m\in{\cal M}$, the optimization problem (\ref{add:29}) should be solved.
For ease of exposition, we set
\begin{align}
F(\theta[t])=&\max_{v_i\in{\cal N}_m,v_m\in{\cal M}}
\left\{\frac{V}{1-\theta[t]}\log\left(\frac{1}{\theta[t]}\right)A_{m,i}(\boldsymbol a[t])\right.\notag\\
&\left.+\frac{V}{1-\theta[t]}B_{m,i}(\boldsymbol a[t])\right\}+C(\boldsymbol a[t]).\label{add:31}
\end{align}
Thus, the optimal model accuracy $\theta[t]$ in time slot $t$ can be obtained by deducing from the following unconstraint problem transformed  from (\ref{add:29}):
\begin{align}
\min_{0\leq \theta[t]\leq 1} F(\theta[t]).\label{add:32}
\end{align}

Following the  SGHS algorithm \cite{pan2010self}, problem (\ref{add:32}) can  be solved by using the  this heuristic optimization approach described in a five-tuple $\text{SGHS}\equiv\langle\text{HMS},\text{HMCR},\text{PAR},\text{BW},\text{NI}\rangle$ consisting of the following components, namely Harmony Memory Size (\text{HMS}), Harmony Memory Consideration Rate (\text{HMCR}), Pitch Adjustment Rate (\text{PAR}), Distance BandWidth (\text{BW}), and the number of Improvisations (\text{NI}).
More precisely, the introduced parameters \text{HMCR}, \text{PAR}, and \text{BW} are adaptively adjusted as the iteration process.
Entries in the harmony memory (HM) are extracted with probability \text{HMCR}.
A large \text{HMCR} is conductive to local search, thereby improving the convergence of the SGHS algorithm, while a small \text{HMCR} increases the diversity of the HM.
In the pitch adjustment process, \text{PAR} indicates that a neighboring value fine-tuned from the value selected in the HM  will be extracted with probability $\text{PAR}\times\text{HMCR}$.
Upon introducing parameter \text{PAR}, the exploitation of the chosen element in the HM is enhanced with a large value of \text{PAR}, whilst the diversity of the HM is increased for a relative small value of \text{PAR}.
Note that the values of \text{HMCR} and \text{PAR} are difficult to determine, due to the contradiction  between local exploitation and global exploration in the SGHS algorithm.
In this paper, without losing generality, we let \text{HMCR} and \text{PAR}  obey normal distributions based on the empirically recommended distribution \cite{geem2001new}, i.e.,
$\text{HMCR}\sim{\cal N}(\mu_{\text{HMCR}}, \sigma_{\text{HMCR}}^2)$ and $\text{PAR}\sim{\cal N}(\mu_{\text{PAR}}, \sigma_{\text{PAR}}^2)$,  respectively.
Moreover, the empirically adjust principle of the \texttt{BW} in iteration $r$ can be inferred from \cite{omran2008global}
\begin{align}\label{ss:34}
\text{BW}(r)=\begin{cases}
\text{BW}_{\max}-\frac{\text{BW}_{\max}-\text{BW}_{\min}}
{\text{NI}} 2r, &\text{if~}r<\frac{\text{NI}}{2},\\
\text{BW}_{\min}, &\text{if~}r\geq \frac{\text{NI}}{2},
\end{cases}
\end{align}
where $\text{BW}_{\min}$ and $\text{BW}_{\max}$ denote the preset minimum and maximum search scopes of the SGHS algorithm.
Let $L$ be the update period of \text{HMCR} and \text{PAR} according to the recording parameters.
For more detailed information concerning the harmony search (HS), the reader is referred to an excellent contribution by Geem and Lee  \cite{geem2001new, lee2005harmony, omran2008global}.
The implementation of the algorithm is presented in Algorithm \ref{algorithm:1}.
\begin{algorithm}[!t]
\caption{Computation  Procedure of the  SGHS Algorithm for (\ref{add:32}) for All Possible Action Policies}\label{algorithm:1}
\begin{algorithmic}[1]
\newcommand{\algorithmicInit}{\textbf{Initialization:}}
\newcommand{\algorithmicIter}{\textbf{Iteration:}}
\newcommand{\algorithmicOutput}{\textbf{Output:}}
\newcommand{\algorithmicBreak}{\textbf{break.}}
\REQUIRE\;

\algorithmicInit\;

{Initialize parameters $\text{HMS}$, $\text{NI}$, $L$, $\text{BW}_{\min}$, $\text{BW}_{\max},$ $\mu_{\text{HMCR}}$, $\sigma_{\text{HMCR}}^2$, $\mu_{\text{PAR}},$ and $\sigma_{\text{PAR}}^2.$}
{Set the HM and assess it using $F(\theta[t]).$}
{Set $r=1, t_L-1,$ and $l_1, l_2, l_3 \in (0,1).$}

\algorithmicIter
\WHILE{$r\leq \text{NI}$}
\STATE Generate \text{HMCR} and \text{PAR} using the mentioned norm distributions and compute $\text{BW}(r)$ according to (\ref{ss:34}).
\IF{$l_1<\text{HMCR}$ }
\STATE{$\theta^{+}[r]=\theta_h[r]\pm l_3\times \text{BW},$  where $h\in \{1, 2, \ldots, \text{HMS}\}.$}
\IF{$l_2<\text{PAR}$}
\STATE{$\theta^{+}[r]=\theta^{\dag}[r],$ where $\theta^{\dag}[r]$ is the best element in the HM.}
\ENDIF
\ELSE
\STATE{$\theta^{+}[r]=\theta_{\min}[r]+l_3\times(\theta_{\max}[r]-\theta_{\min}[r])$}
\ENDIF
\IF{$F(\theta^{+}[r])<F(\theta^{\ddag}[r])$ }
\STATE{Substitute $\theta^{\ddag}[r]$ in HM to $\theta^{+}[r]$, where $\theta^{\ddag}[r] $ is the worst element in the HM.}
\ENDIF
\IF{$t_L=L$ }
\STATE{Update the $\mu_{\text{HMCR}}$ and $\mu_{\text{PAR}}$ by averaging the recorded values \text{HMCR} and \text{PAR}.}
\STATE{Reset $t_L=1$}
\ELSE
\STATE{$t_L=t_L+1$}
\ENDIF
\STATE{$r=r+1$}
\ENDWHILE

\algorithmicOutput {~The optimal local model accuracy $\theta^*[r]$ in the HM as evaluated by $F(\theta[r]).$}
\end{algorithmic}
\end{algorithm}

The optimal solution to (\ref{add:19}) can be obtained by iteratively performing {\bf\em Steps 1-4} and Algorithm  \ref{algorithm:1}.
However, the above optimal solution with an exponential computational complexity is prohibitive in practice settings with a large set of workers.
This motivates the pursuit of computationally efficient near-optimal solutions.
Thus, in the following, we further develop a low-complexity method to obtain a suboptimal solution to problem (\ref{add:19}) through exploiting the local information of each RC $v_m\in{\cal M}$.
Drawing on the matching game theory \cite{gu2015matching}, we employ a matching game based distributed algorithm to solve problem (\ref{add:19}), because of its ability to tackle mixed integer optimization problems, which will be presented in the ensuing section.

\section{Distributed LRef-FedCS Method}
In this section, we present the formulation of the LRef-FedCS process as a matching game, where URs minimize their individual costs in a distributed manner.
We analyze the existence of a stable solution and demonstrate that the complexity of the analysis algorithm is polynomial.
\subsection{Matching Game Based Distributed Method Design}
In the previous section, the FL server coordinates the optimal LRef-FedCS in a centralized manner.
However, we cannot design a decentralized and scalable algorithm using {\bf\em Steps 1-4}, which use an objective function that requires global information of all clients (URs and UnRCs).
Due to the distributed nature of mobile IoT devices, each RC is only aware of the UnRCs within its individual sensing area without additional signaling and overhead.
In the following, we denote the set of UnRCs located in the sensing area of $v_m$ by ${\cal N}_m^{\text{nei}}[t]$,  namely  ${\cal N}_m^{\text{nei}}[t]:=\{v_n\in{\cal N}: \text{dis}(v_m, v_n)\leq s_{\max}~\&~w_{m,n}>0\}$, where $s_{\max}$ is the sensing distance of RCs.
Therefore, at the {\em renewal time instance} $\tau[t]$, we define $\tilde{\boldsymbol A}_m[t]=\{\tilde{\boldsymbol a}_m[t]\}$ as the action set of $v_m$, where $\tilde{\boldsymbol a}_m[t]=\{\phi_m[t],  {\tilde{\pmb\alpha}}_m[t]\}$ is the action policy based on the local information, and $\tilde{\pmb\alpha}_{m}[t]=\{\alpha_{m,i}[t], i\in{\cal N}_m^{\text{nei}}[t]\cup\{v_m\}\}.$
Based on the notation of $\tilde{\boldsymbol a}_m[t]$, we further define the individual utility $U_m(\tilde{\boldsymbol a}_m[t], \theta[t])$ of each RC $v_m$ as follows
\begin{equation}\label{add:34}
\begin{aligned}
U_m(\tilde{\boldsymbol a}_{m}[t], \theta[t])=&-\left\{
\frac{V}{(1-\theta[t])}\log\left(\frac{1}{\theta[t]}\right)A_{m,i}(\tilde{\boldsymbol a}_{m}[t])\right.\\
&\left.+\frac{V}{(1-\theta[t])}B_{m,i}(\tilde{\boldsymbol a}_{m}[t])
+C_m(\tilde{\boldsymbol a}_m[t])
\right\},
\end{aligned}
\end{equation}
where $A_{m,i}(\tilde{\boldsymbol a}_m[t])$ and $B_{m,i}(\tilde{\boldsymbol a}_m[t])$ can be deduced from (\ref{a:3}) and (\ref{a:4}) by substituting $\tilde{\boldsymbol a}_m[t]$ for ${\boldsymbol a}[t]$, respectively.
Further, $C_m(\tilde{\boldsymbol a}_m[t])$ quantifies the penalty component for $v_m$ based on the virtual queues $\tilde{\gamma}_m$ and $\tilde{z}_n$, derived from (\ref{add:18}) and (\ref{add:18-1}).
These computations employ ${\cal N}_m^{\text{nei}}[t]$ and ${\cal N}_{\text{act},m}^{\text{nei}}[t]$, enabling $C_m(\tilde{\boldsymbol a}_m[t])$ to be formally expressed as
\begin{align}
C_m(\tilde{\boldsymbol a}_m[t])=&\tilde\gamma_m[t]\left[\Delta-\sum_{v_i\in{\cal N}_m^{\text{nei}}[t]}\hspace{-0.2em}\alpha_{m,i}[t]\right]
+\sum_{v_n\in{{\cal N}_{\text{act},m}^{\text{nei}}}[t]}\hspace{-0.2em}\tilde z_n[t]\notag\\
&\times\underbrace{\left[\alpha_{m,n}[t]{\cal I}_{{\cal N}_{\text{act},m}^{\text{nei}}[t]}(n)\left(R_{n,\min}^{\text{c2c}}[t]-R_n^{\text{c2c}}[t]\right)\right]}_{[\eqref{add:35}-1]}.\label{add:35}
\end{align}
Given the above individual utility of each RC $v_m$, we can formulate the following problem, as derived directly from (\ref{add:19}) by substituting $\tilde{\boldsymbol a}_m[t]$ and ${\cal N}_m^{\text{nei}}[t]$ for ${\boldsymbol a}[t]$ and ${\cal N}_m^{\text{tru}}[t]$, respectively
\begin{align}
&\max_{\tilde{\boldsymbol a}_m[t]\in\tilde{\boldsymbol A}_m[t]}~U_m(\tilde{\boldsymbol a}_m[t], \theta[t])\label{a:6}\\
&\hspace{1.5em}\text{s.t.~}(\ref{add:12}) \text{~and~} (\ref{add:17}\text{b})-(\ref{add:17}\text{f}).\tag{\ref{a:6}a}
\end{align}

In problem (\ref{a:6}), the objective is reduced to maximize  the utility of each RC by LRef-FedCS in parallel.
A nature framework for studying such a combinatorial problem is offered by the matching theory which can achieve a distributed solution.
Following the two-sided matching game \cite{kazmi2017mode}, we formulate the LRef-FedCS as a matching game between the RCs and UnRCs, and then present a matching algorithm that can find a stable matching which is a key concept for a matching game.
As can be seen from constraints (\ref{add:17}e) and (\ref{add:17}f), we assume that each RC forms a set that can recommend a single UnRC.
However, for the recommended UnRC, the computation time required  by the UnRC should be below the tolerable predefined upper bound, i.e., constraint (\ref{add:12}).
Similarly, each UnRC also forms a set to accommodate an RC among all the RCs.
Therefore, our design amounts to a one-to-one matching given in the tuple ${\cal MG}\equiv\langle{\cal M}, {\cal N}, \succ_{\cal M}, \succ_{\cal N}\rangle$, where $\succ_{\cal M}:\triangleq\{\succ_{v_m}\}_{v_m\in{\cal M}}$ and $\succ_{\cal N}:\triangleq \{\succ_{v_n}\}_{v_n\in{\cal N}}$ represent the sets of preference relation of the RCs and UnRCs, respectively.
Formally, we define the relation ``prefer'' for $v_m$ between $v_n$ and $v_{n'}$ in the following definition.
\begin{definition}\label{def:1}
 At the {\em renewal time point} $\tau[t]$, $v_m$ prefers $v_n$ to $v_{n'}$, if $U_m(\phi_m[t], \alpha_{m,n}[t], \theta[t])>U_m(\phi_m[t], \alpha_{m, k'}[t], \theta[t])$, denoted by $v_n\succ_{v_m}^{[t]} v_{n'}$,  $\forall v_n, v_n' \in {\cal N}_m^{\text{nei}}[t], v_n\neq v_n'$.
And $v_n$ prefers $v_m$ to $v_{m'}$, if $U_m(\phi_m[t],  \alpha_{m,n}[t], \theta[t])>U_{m'}(\phi_{m'}[t], \alpha_{m',n}[t], \theta[t])$, denoted by $v_m\succ_{v_n}^{[t]}v_{m'}$, $\forall v_m, v_m'\in{\cal M}_n[t]=\{v_m:~v_n\in{\cal N}_m^{\text{nei}}[t], \forall v_m\in{\cal M}\}, v_m\neq v_m'$.
${\cal M}_n[t]$ represents the set of RCs that can collect information regarding UnRC $v_n$ in the training round $t$.
\end{definition}

Note that the set of preference relations of UnRCs over $v_m\in{\cal M}$ is also based on the preference function $U_{m}(\tilde{\boldsymbol a}_m[t], \theta[t])$.
This is because that the long-term QoS requirements of the UnRCs are embedded into the utility function.
Thus, obtaining the preference value of $v_n\in{\cal N}$ over $v_m\in{\cal M}$ also relies on the utility $U_m(\tilde{\boldsymbol a}_m[t], \theta[t])$, mainly due to the part of [\eqref{add:35}-1].
Moreover, we define the preference list of each $v_m\in{\cal M}$ (respectively, $v_n\in{\cal N}$) over $v_n\in{\cal N}$ (respectively, $v_m\in{\cal M}$) by ${\cal L}_m^{\text{rc}}$ (respectively, ${\cal L}_n^{\text{unrc}}$), which is ranked by the preference value in  non-increasing order.
However, as shown in Definition \ref{def:1}, the two sets of players in the matching game are the subsets of ${\cal M}$ and ${\cal N}$, respectively.
In order to be consistent with the matching game ${\cal MG}$, the preference list ${\cal L}_m^{\text{rc}}$ of each $v_m\in{\cal M}$  consists of two parts.
The first part is ranked by the preference value in non-increasing order over $v_n, v_n\in{\cal N}_m^{\text{nei}}[t]$, whilst the second part is just set $v_n, v_n\in{\cal N}\setminus{\cal N}_m^{\text{nei}}[t]$.
The preference value corresponding to the elements in the set is set to be $-\infty$.
Similarly, the preference list ${\cal L}_n^{\text{unrc}}$ of each $v_n\in{\cal N}$ comes from two subsets ${\cal M}_n[t]$ and ${\cal M}\setminus {\cal M}_n[t]$, where the preference value of $v_n\in{\cal N}$ over $v_m, v_m\in{\cal M}\setminus {\cal M}_n[t]$ is set to be $-\infty$, and the preference list over ${\cal M}_n[t]$ is ranked in non-increasing order.
Accordingly, a matching game based  decentralized method is developed, since only local information is used for each $v_m\in{\cal M}$ and $v_n\in{\cal N}$.
In addition, the matching game based decentralized method has a stable strategy that is a key concept in matching theory \cite{roth2008deferred} and can be defined as follows.
\begin{definition}
A matching $\cal MG$ is stable for each time slot $t$ if there is no blocking pair $(v_m,v_n)$, where $v_m\in{\cal M}, v_n\in{\cal N}$, such that $v_n\succ_{v_m}^{[t]}{\cal MG}(v_m)$ and $v_m\succ_{v_n}^{[t]}{\cal MG}(v_n)$, where ${\cal MG}(v_m)$ and ${\cal MG}(v_n)$ represent the current partners of $v_m$ and $v_n$, respectively.
\end{definition}
\begin{algorithm}[!t]
\caption{Distributed  Method for Solving Problem (\ref{add:19}) on Each Training Task}\label{algorithm:3}
\begin{algorithmic}[1]
\newcommand{\algorithmicInit}{\textbf{Initialization:}}
\newcommand{\algorithmicIter}{\textbf{Iteration:}}
\newcommand{\algorithmicOutput}{\textbf{Output:}}
\newcommand{\algorithmicBreak}{\textbf{break.}}
\REQUIRE\;
\algorithmicInit\;

{Initialize the local information $\tilde{\boldsymbol S}_m[t]=\{{\boldsymbol H}_m[t], {\boldsymbol W}_m[t], {\boldsymbol Q}_m[t], {\boldsymbol G}_m[t]\}$ of each $v_m\in{\cal M}$.}

\% \underline{\em{Stage I: Matching Process:}}

\STATE{{\bf Input:} RCs' preference list ${\cal L}^{\text{rc}}[t]$ and UnRCs' preference list ${\cal L}^{\text{unrc}}[t]$ in time slot $t$.}
\STATE{{\bf Output:} System-optimal stable matching ${\cal MG}$.}
\STATE{{\bf Metodo:}}
\STATE{Set up RCs' preference lists as ${\cal L}_m^{\text{rc}}$, $\forall v_m\in{\cal M}$;}
\STATE{Set up UnRCs' preference lists as ${\cal L}_k^{\text{unrc}}$, $\forall v_n\in{\cal N}$;}
\STATE{Set up a list unmatched RCs ${\cal UM}=\{v_m, \forall v_m\in{\cal M} \};$  }
\WHILE{${\cal UM}$ is not empty}
\STATE{$v_m^{\text{rc}}$  proposes to the UnRC that locates first in its list, $\forall v_m \in{\cal UM}$;}
\IF{$v_n\in{\cal N}$ receives a proposal from $v_{m'},$ and $v_{m'}$ is more preferred than the current hold $v_m$}
\STATE{$v_n$ holds $v_{m'}$ and rejects $v_m$;}
\STATE{$v_{m'}$ is removed from ${\cal UM}$ and $v_m$ is added into ${\cal UM};$ }
\ELSE
\STATE{UnRC rejects $v_{m'}$ and continues holding $v_m$;}
\ENDIF
\ENDWHILE
\STATE{Output optimal matching ${\cal MG}.$}

\% \underline{\em{Stage II: Local Model Accuracy Optimization:}}
\STATE{Run Alg. \ref{algorithm:1} for problem (\ref{add:32}) to obtain $\theta[t]$.}
\end{algorithmic}
\end{algorithm}

\begin{theorem}
Given a fixed  balance parameter $V$, in each time slot $t$, the stable solution $({\cal MG}^{*}[t], \theta^{*}[t])$ produced by Algorithm \ref{algorithm:3} is a local solution to the optimization problem  (\ref{a:6}).
\end{theorem}
\begin{proof}
Please see Appendix A.
\end{proof}

\subsection{Algorithm Summary and Discussions}
In our method, a stable solution ensures that no matched RC will benefit from deviating from its client selection.
The output of the matching game-based distributed method is the action policy $\tilde{\boldsymbol a}_m[t]$ of each $v_m\in{\cal M}$ that maximizes its utility $U_m$.
The the pseudo code to find a stable matching $\cal MG$ between the RCs and UnRCs is given in Algorithm \ref{algorithm:3}.
Furthermore, the computational complexity of {\em Stage I} in Algorithm 2 can be shown to be \ref{algorithm:3} by ${\cal O}(\max_{v_m\in{\cal M}, v_n\in{\cal N}}|{\cal N}_m^{\text{nei}}[t]|\cdot|{\cal M}_n[t]|).$
In fact, we denote the complexity of building the preference list ${\cal L}_m^{\text{rc}}[t]$ for each $v_m$ by ${\cal O}(|{\cal N}_m^{\text{nei}}[t]|\log(|{\cal N}_m^{\text{nei}}[t]|))$.
Consequently,  the complexity of building the preference list ${\cal L}_n^{\text{rc}}[t]$ at each $v_n$ is ${\cal O}(|{\cal M}_n[t]||{\cal N}_m^{\text{nei}}[t]|\log(|{\cal M}_n[t]||{\cal N}_m^{\text{nei}}[t]|))$.
So, the complexity of {\em Stage I} in Algorithm \ref{algorithm:3} is upper bounded by ${\cal O}(\max_{v_m\in{\cal M}, v_n\in{\cal N}}|{\cal N}_m^{\text{nei}}[t]||{\cal M}_n[t]|).$

\section{Simulation Results and Discussions}
\subsection{Simulation Environments and Settings}
In this section, the short- and long-term performances of the proposed centralized  ({\bf\em Steps 1-4}) and distributed methods are evaluated and compared.
To keep the complexity of the simulations tractable while considering a significantly loaded system, we focus on a circular community with an area $\pi(50\times50)\text{~m}^2$ centered around its server, where $M$ RCs are randomly distributed, each with a sensing radius of $s_{\max}=18\text{~m}$.
To model a relatively real-world mobility environment, the Gauss-Markov Mobility Mode \cite{camp2002survey} originally constructed to simulate personal communications systems \cite{liang1999predictive} is employed in our simulations.
In addition, the social relationship ${\boldsymbol W}$ can be established through online social networks, whose entries $w_{m,n}$ are quantified by the average communication frequency between the RC $v_m$ and UnRC $v_n$ \cite{gao2019dynamic}.

Based on the communication phase described in Section V.A,  it is necessary to carefully configure the simulation parameters.
First, we consider a wireless FL with carrier frequency $2.3\text{~GHz}$ and bandwidth $0.2\text{~MHz}.$
We set the path loss exponents for each individual link and the path loss at the reference distance $1\text{~m}$  to be $3$ and $30$ dBm,  respectively, and adopt  Rayleigh fading to reproduce multipath effects \cite{gao2021reflection}.
We assume that perfect channel information is known at the receivers pilot signals \cite{gao2021reflection}.
Throughout the simulations, unless otherwise specified, we adopt the parameters reported in Table \ref{table:2}, where, for simplicity, each UnRC is assumed to have the same long-term QoS requirement, i.e.,  $R_{\min}^{\text{c2c}}$, and each potential participant is also assumed to have the same computation time constraint in one local iteration, i.e.,  $\text{T}_{\max}^{\text{cmp}}$.
Moreover, all the RCs are assumed to have the same CPU frequency $f_m=f_{\text{rc}}$ and transmit power budget $p_m^{\max}=p_{\text{rc}}^{\max}$.
Likewise,  the same CPU frequency $f_n=f_{\text{urc}}$ and the maximum transmit power $p_n^{\max}=p_{\text{urc}}^{\max}$ are imposed for all UnRCs.
In addition, we set $\lambda^t={1}/{6}$ and $\lambda^e={5}/{6}$.
\begin{table}[!t]
\centering
\caption{General System Parameters}\label{table:2}
\begin{tabular}{p{0.3\textwidth}p{0.13\textwidth}}
  \hline
  {\bf Simulation Parameter} & {\bf Value}  \\
  \hline
  Total number of RCs, UnRCs, $M, N$ & $10, 60$  \\

  Total bandwidth   & $0.2\text{MHz}$\\

  Carrier frequency & $2.3\text{GHz}$\\

Transmit power budget of RCs, $p_{\text{rc}}^{\max}$& $0.5\text{W}$\\

  Transmit power budget of UnRCs, $ p_{\text{urc}}^{\max}$& $ 0.3\text{W}$\\

  CPU frequency of RCs,  $f_{\text{rc}}$& {$2\times10^8\text{cycles/byte}$}\\

  CPU frequency of  UnRCs, $f_{\text{urc}}$ &$2\times 10^7\text{cycles/byte}$\\

 Switched capacitance, $\rho_n$ &$10^{-27}$\\

  Constant value, $\zeta$ & $3$\\

  Number of training samples, $Q_i$& $10000$\\

 Processing density, $b_i$ & $10\text{cycles}$\\

Computation time constraint, $\text{T}_{\max}^{\text{cmp}}$&$0.1\text{s}$\\

\midrule

HM size, \text{HMS}& $5$\\

Number of improvisations, \text{NI}&$300$\\

Maximum distance bandwidth, $\text{BW}_{\max}$& $0.5$\\

Minimum distance bandwidth, $\text{BW}_{\min}$& $5\times 10^{-4}$\\

Mean HMCR, PAR, $\mu_\text{HMCR}, \mu_{\text{PAR}}$& $0.95$, $0.3$\\

Variance of HMCR, PAR, $\sigma_{\text{HMCR}}, \sigma_{\text{PAR}}$& $0.01, 0.05$\\
  \hline
\end{tabular}
\end{table}
\begin{figure}
\begin{minipage}[t]{0.235\textwidth}
\centering
\includegraphics[width=\textwidth]{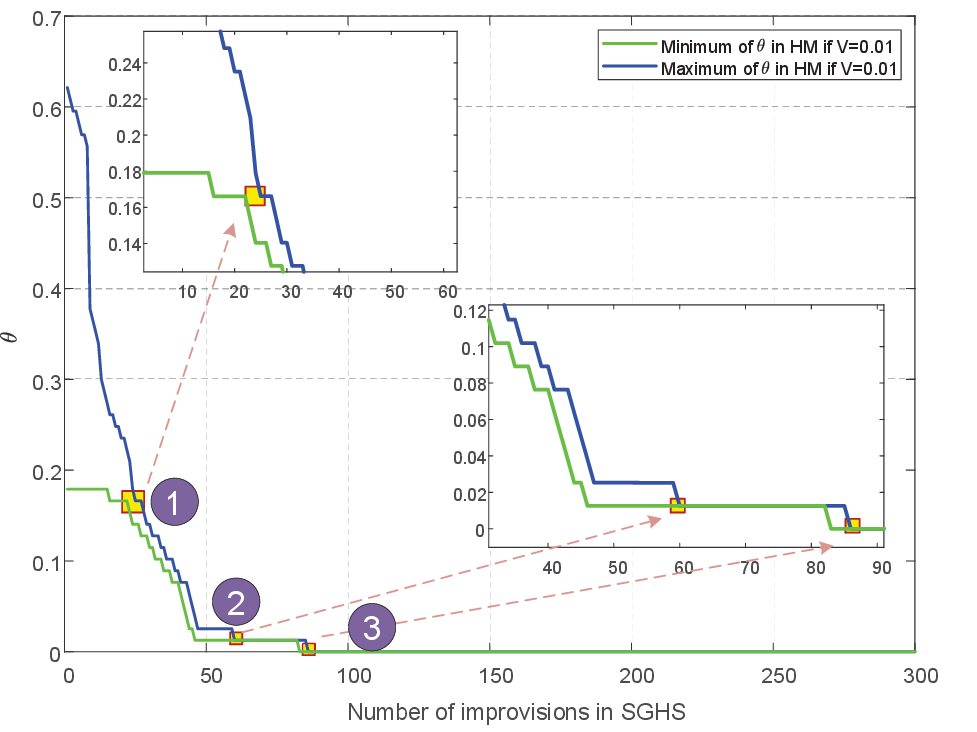}
\caption{Convergence of SGHS versus the number of improvisations.}
\label{fig:convergence}
\end{minipage}
\begin{minipage}[t]{0.235\textwidth}
\centering
\includegraphics[width=\textwidth]{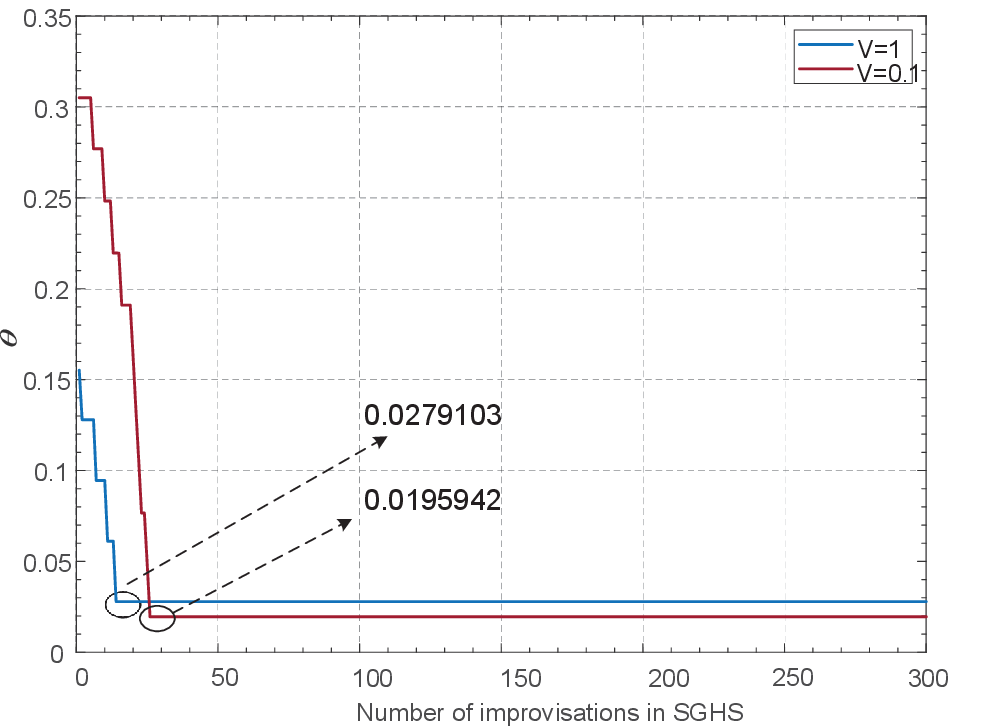}
\caption{Local model accuracy versus the number of improvisations.}
\label{fig:sghs}
\end{minipage}
\end{figure}

\subsection{Algorithm Discussion}
{\bf\em 1) Validation of convergence of SGHS:}
Fig. \ref{fig:convergence} depicts the convergence of SGHS (Alg. \ref{algorithm:1}) in one training task.
We draw the maximum and minimum values of the HM to characterize the value range and convergence of the HM through solving the unconstrained problem (\ref{add:32}) with $V=0.01$ and randomly selected feasible action policy ${\boldsymbol a}$, varying the number of improvisations from $0$ to $300$.
As can be seen from the figure,  the gap between the two lines decreases when the number of improvisations increases.
More precisely, for the \text{NI} values to the left of point \textcircled{1}, i.e., $\text{NI}\in [0, 20]$, the gap between the two curves in increasing iterations is significantly narrowed.
However, for the \text{NI} values to the right of point \textcircled{3}, it is difficult to distinguish the two curves, and they converge to $\theta^{*}=1.69236\times 10^{-5}$ quickly.
It can be shown that for the iterations between points $\textcircled{1}$ and $\textcircled{3}$, the infection points and gaps between the two curves like point $\textcircled{2}$ appear, which means that the HM is updating and adding the new $\theta$ frequently.
The effect of the balance parameter $V$ in the  Lyapunov-based algorithm on the optimal solution to the SGHS process in one global round is shown in Fig. \ref{fig:sghs}.
As can be observed from the figure, for any given feasible action policy ${\boldsymbol a}$, the optimal value of the local model accuracy increases with  $V$.
This is attributed to the terms of $VA_i^{(m)}(\boldsymbol a)$ and $VB_i^{(m)}(\boldsymbol a)$ in  (\ref{add:31}), which can be regarded as {\bf\em coefficients} for $[\log(1/\theta)/(1-\theta)]$ and $[1/(1-\theta)]$, respectively.
Specifically, when a much larger value of $V$ is provided, solving (\ref{add:32}) results in a larger optimal $\theta$, as the associated {\bf\em coefficients} of $F(\theta)$ exhibit relatively higher magnitudes.
\begin{figure}[!t]
\begin{minipage}[t]{0.265\textwidth}
\centering
\includegraphics[width=\textwidth]{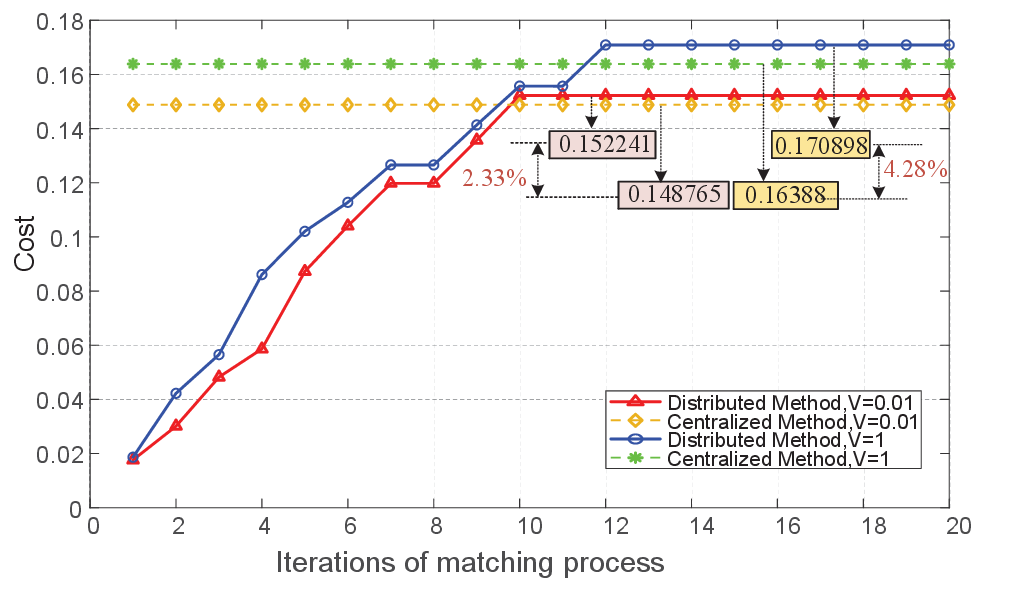}
\caption{System cost as a function of the iteration step of the matching process.}
\label{fig:gale}
\end{minipage}
\begin{minipage}[t]{0.215\textwidth}
\centering
\includegraphics[width=\textwidth]{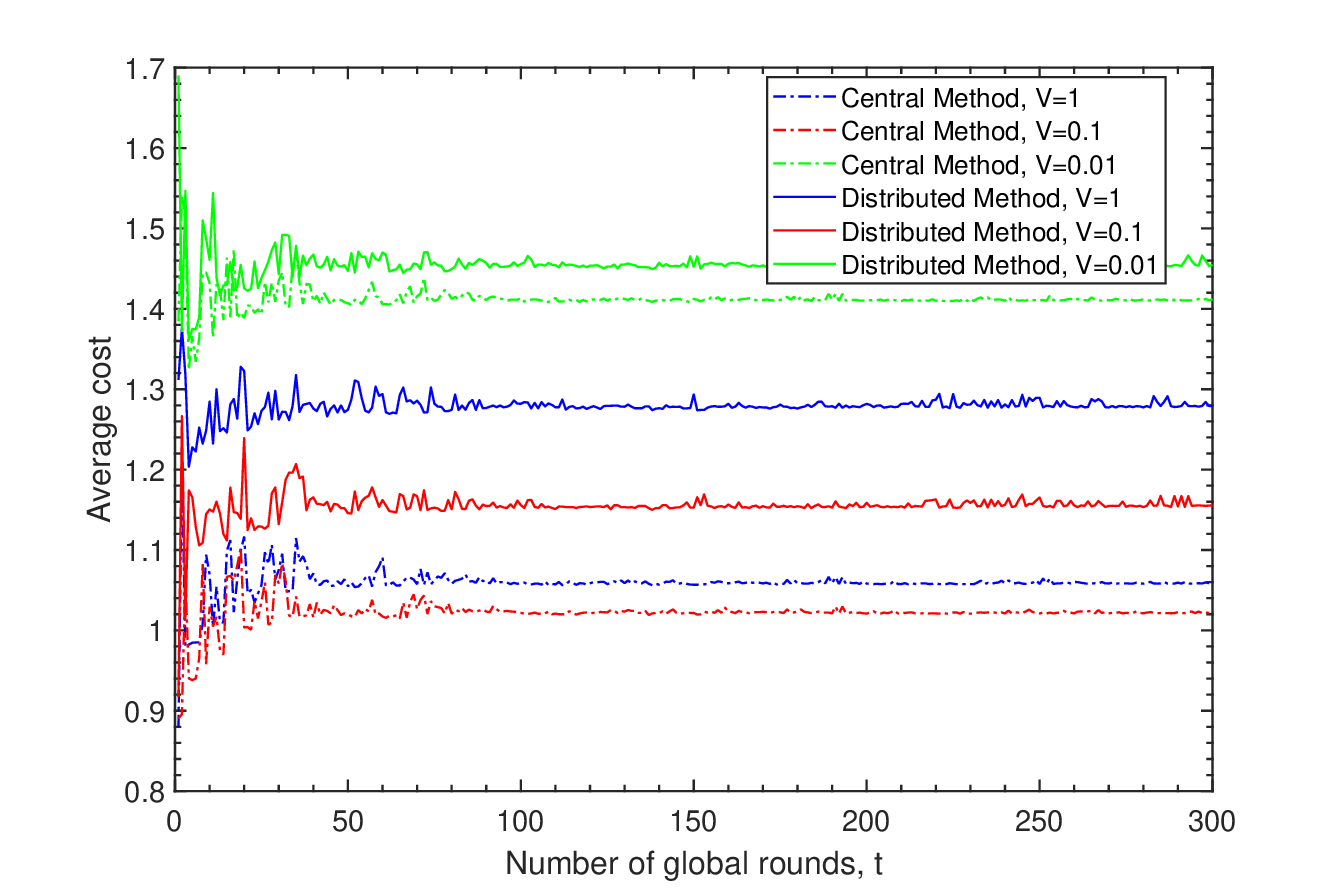}
\caption{Time-average cost versus the number of global rounds.}
\label{fig:avecost}
\end{minipage}
\end{figure}

{\bf\em 2) Tightness of the matching game-based distributed method:}
To assess the gap between  the solution yielded by the distributed method and the optimal solution obtained by the centralized method ({\bf\em Steps 1-4}),  Fig. \ref{fig:gale} plots the system cost as a function of the iteration for the matching process in Algorithm \ref{algorithm:3}.
For any given value of $V$, we first solve (\ref{add:19}) using the centralized method  and then solve the optimization problem using the distributed method under the same mentioned system settings.
As can be seen from  Fig. \ref{fig:gale}, the matching game-based distributed method can produce solutions close to the global optimal one, but the gap depends on the value of $V$.
We can observe from the results in Fig. \ref{fig:gale}
that the distributed method can perform closely to the global optimal solution with only $2.33\%$ bias within $10$ iterations for $V=0.01$.
This indicates that if the value of $V$ is set appropriately,  the proposed distributed method can yield solutions close to the optimal one rapidly.

{\bf\em 3) Lyapunov-based algorithm characteristics:}
We identify the characteristics of the Lyapunov-based algorithm with respect to the renewal times (FL announced training tasks) under different values of the balance parameter $V$.
To this end, the time-average cost performance of the centralized and distributed methods assisted Lyapunov algorithms are evaluated and compared by varying $t-1$ from $0$ to $300$.
Fig. \ref{fig:avecost} shows the impact of the balance parameter $V$ on the time-average cost for both the distributed and centralized methods assisted Lyapunov algorithms.
As expected,  for both the centralized and distributed methods, the time-average cost is asymptotically stable as it evolves over time.
In addition to our observations in Fig. \ref{fig:avecost} with respect to $t$, the performance bias between the centralized and distributed methods  increases with $V$.
This phenomenon is implied in (\ref{add:31}).
\begin{figure}[!t]
	\centering
 \begin{minipage}[t]{0.48\textwidth}
 \centering
	\includegraphics[width=1\textwidth]{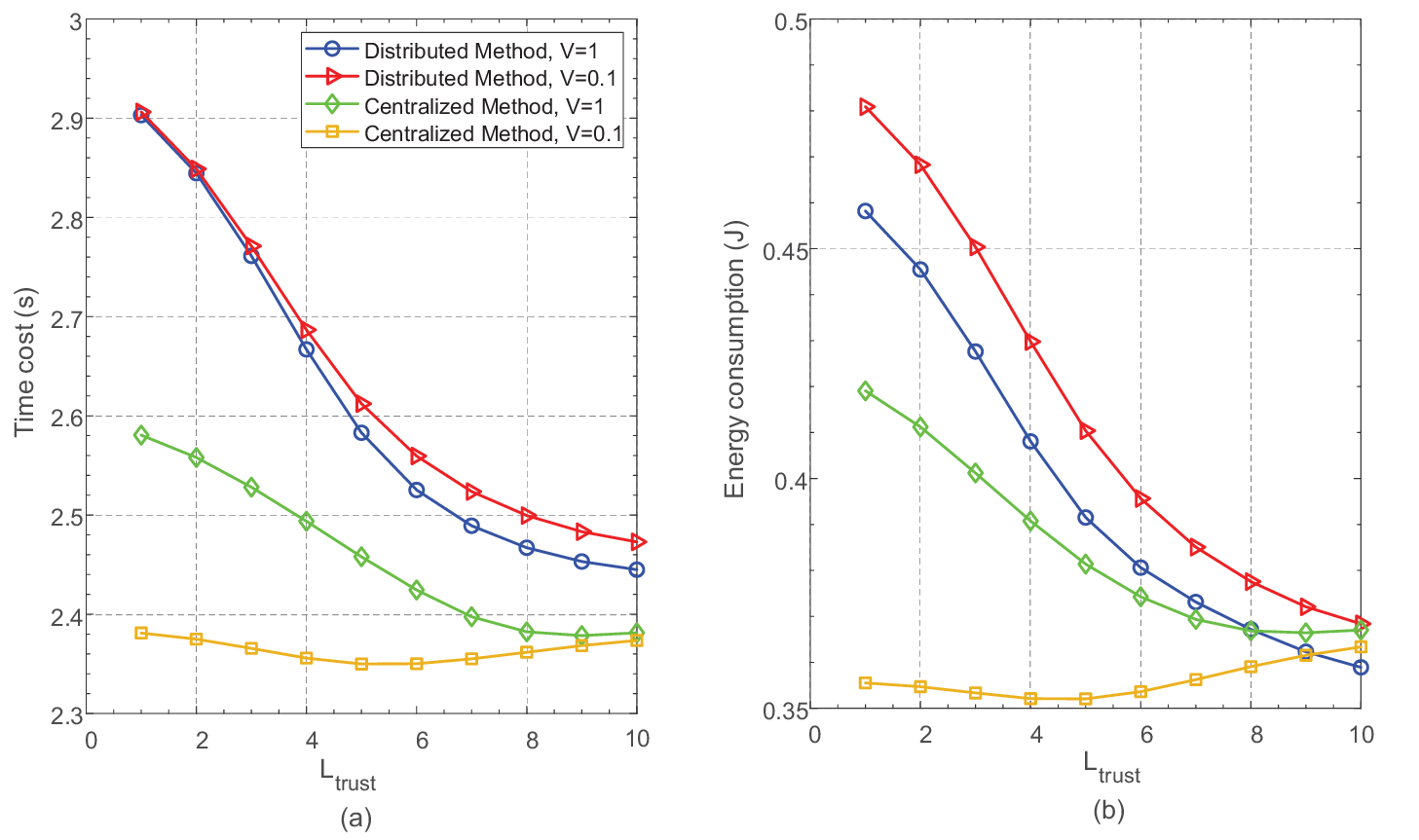}
	\caption{Time-average training time cost and  energy consumption versus $L_{\text{trust}}$.}
	\label{fig:trust}
 \end{minipage}
\end{figure}

{\bf\em 4) Impact of trust value distribution:}
To emphasize the impact of the trust distribution on the LRec-FedCS process, we define a new parameter, quantifying the clients heterogeneity regarding to trust value, as follows
\begin{align}
L_{\text{trust}}:\triangleq\frac{\min\limits_{v_m\in{\cal M}}\max\limits_{v_i\in{\cal N}_m}w_{m,i}}
{\max\limits_{v_m\in{\cal M}}\min\limits_{v_i\in{\cal N}_m}w_{m,i}}.
\end{align}
We see that higher values of $L_{\text{trust}}$ indicate higher levels of client heterogeneity.
For example, $L_{\text{trust}}=1$ can be considered as a low heterogeneity level due to balanced trust distribution, which means that the lower limit of the maximum sequence (numerator) is the same as the upper limit of the minimum sequence (denominator).
In our simulations, to vary $L_{\text{trust}}$, the trust level $w_{m,k}$ between the RCs (belong to set ${\cal M}$) and UnRCs (located in set ${\cal N}$) is generated with the fraction $\frac{\min_{v_m\in{\cal M},v_n\in{\cal N}}w_{m,n}}{\max_{v_m\in{\cal M},v_n\in{\cal N}}w_{m,n}}\in\{1, 0.9, 0.8, \ldots, 0.1\}$,  but the maximal value of the trust level $\max_{v_m\in{\cal N},v_n\in{\cal N}}w_{m,n}$ is kept at the same value one.
Figs. \ref{fig:trust}(a) and (b) show the influence of $L_{\text{trust}}$ on the training time cost and the average energy consumption for both centralized and distributed methods.
As can be seen from the figures,  the time-average training time cost and energy consumption of the proposed methods decrease for a given value of $V$, as  $L_{\text{trust}}$ increases.
This can be explained as follows.
According to the detailed expressions presented in Section V.A, both the time and energy consumption are functions of trust level, considering a given value of $p_{\text{rc}}^{\max}, p_{\text{urc}}^{\max}, B$, as well as for a given value of $f_{\text{rc}}$ and $f_{\text{urc}}$.
In addition, as the parameter $L_{\text{trust}}$ increases, it leads to a more refined segmentation of communication and computing resources.
This finer segmentation enhances the potential for achieving system cost efficiency.

\subsection{Performance Comparison}
To gain insights into the distributed and centralized methods assisted Lyapunov-based algorithms, we compare the time-average cost performance of the proposed algorithms with those of the following methods:
\begin{itemize}
\item Greedy+SGHS method: select clients according to the highest trust levels and obtain local model accuracy via SGHS.
\item Random+SGHS method: randomly select clients and obtain local model accuracy via SGHS.
\item SQoS+SGHS method: select clients from a subset of UnRCs that currently have their own QoS requirements (i.e., clients located in ${\cal N}_{\text{act},m}^{\text{tru}}[t]$) and obtain local model accuracy via SGHS.
\item Greedy+Random method: select clients according to the highest trust levels and randomly select the local model accuracy.
\item Random+Random method: randomly select clients and the local model accuracy.
\item SQoS+Random method:  select  clients from a subset of SWs that currently have their own QoS requirements and randomly select the local model accuracy.
\end{itemize}
The performance comparison is based on the assumption that the aforementioned methods use the same feasible action set that is determined by {\bf\em Step 1} at each {\em renewal time point} $\tau[t]$.
\begin{figure}
\begin{minipage}[t]{0.25\textwidth}
\centering
\includegraphics[width=\textwidth]{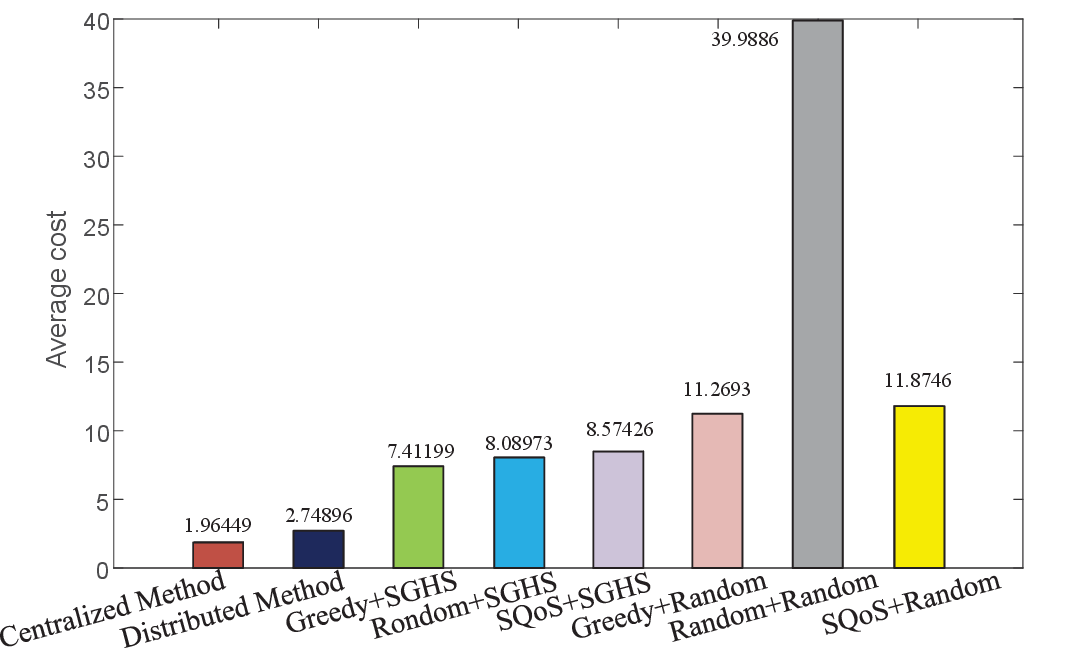}
\caption{Comparison of time-average cost for different methods.}
\label{fig:compare}
\end{minipage}
\begin{minipage}[t]{0.23\textwidth}
\centering
\includegraphics[width=\textwidth]{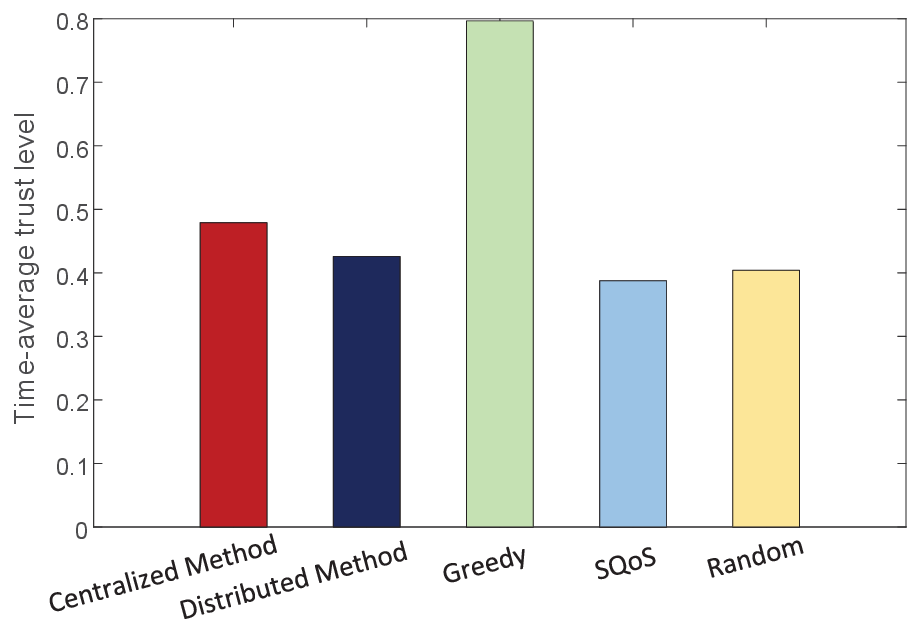}
\caption{Comparison of time-average trust level of  recommended UnRCs.}
\label{fig:compare_trust}
\end{minipage}
\end{figure}

Fig. \ref{fig:compare} compares the time-average cost achieved by all the above methods with $V=1$ and $\text{NI}=300.$
As can be observed from the figure, the proposed distributed method significantly outperforms the  methods in time-average cost other than the centralized method.
This is not only due to the difference between the local and global information, but also because of the sequential nature of the decision-making process, where the optimal solution produced by the centralized method is to solve $\min\limits_{\theta[t]}\overbrace{\min_{\boldsymbol a[t]}\max_{v_i\in{\cal N}_m[t], v_m\in{\cal M}}{\cal G}_{m,i}({\boldsymbol a}[t], \theta[t])}^{\clubsuit}$, and the distributed method is to solve $\min\limits_{\theta[t]}\overbrace{\max_{v_i\in{\cal N}_m^{\text{nei}}[t], v_m\in{\cal M}}\min_{\tilde{\boldsymbol a}[t]}{\cal G}_{m,i}(\tilde{\boldsymbol a}[t], \theta[t])}^{\spadesuit}$.
Thus, the corresponding  coefficients $A_{m,i}$ and $B_{m,i}$ in the output $F(\theta)$ of $\spadesuit$  are less than the coefficients associated with $\clubsuit$.
For the outcomes of $\clubsuit$ and $\spadesuit$ with a given value of $V$, the corresponding $VA_{m,i}, VB_{m,i}$, $C$, and $C_{m}$ are constants.
The larger the constants, the greater the incurred cost of (\ref{add:32}).
In addition, there is a long-span of average cost between the Greedy, Random, SQoS methods in conjunction with SGHS and the proposed distributed method.
This is to be expected because the distributed methods yield the smallest coefficient for the same value of $V$.
Moreover, when the local model accuracy is randomly chosen, it leads to a substantial increase in the incurred cost, which is further exacerbated by inappropriate settings of $\theta$,  as depicted Random+Random in Fig. \ref{fig:compare}.

Here, to demonstrate the superiority of the LRef-FedCS framework, we further compare the reliability performance of FL and quantify it by the time average of the successfully recommended clients' trust values.
It is clear from  Figs. \ref{fig:compare} and \ref{fig:compare_trust} that the highest FL reliability is produced by the Greedy method at the expense of the more system cost.
It is noted that, the FL reliability of the SQoS method is relatively small, which is essentially due to the fact that this method priorities the QoS of UnRCs rather than the system cost.
According to the designed trust-aided diverse resource trading principles $\Xi(\cdot)$ and $\Pi(\cdot)$ respectively described in (\ref{b:1}) and (\ref{b:2}), the proposed centralized and distributed methods  balance well between FL reliability and system cost.
According to the bandwidth segmentation of the RCs and the computational-communication resource allocation obeying the designed principles $\Xi(\cdot)$ and $\Pi(\cdot)$, respectively, both the proposed centralized and distributed LRef-FedCS strike an elegant between cost-efficiency and FL reliability.
\begin{figure}[!t]
	\centering
	\includegraphics[width=3.75in]{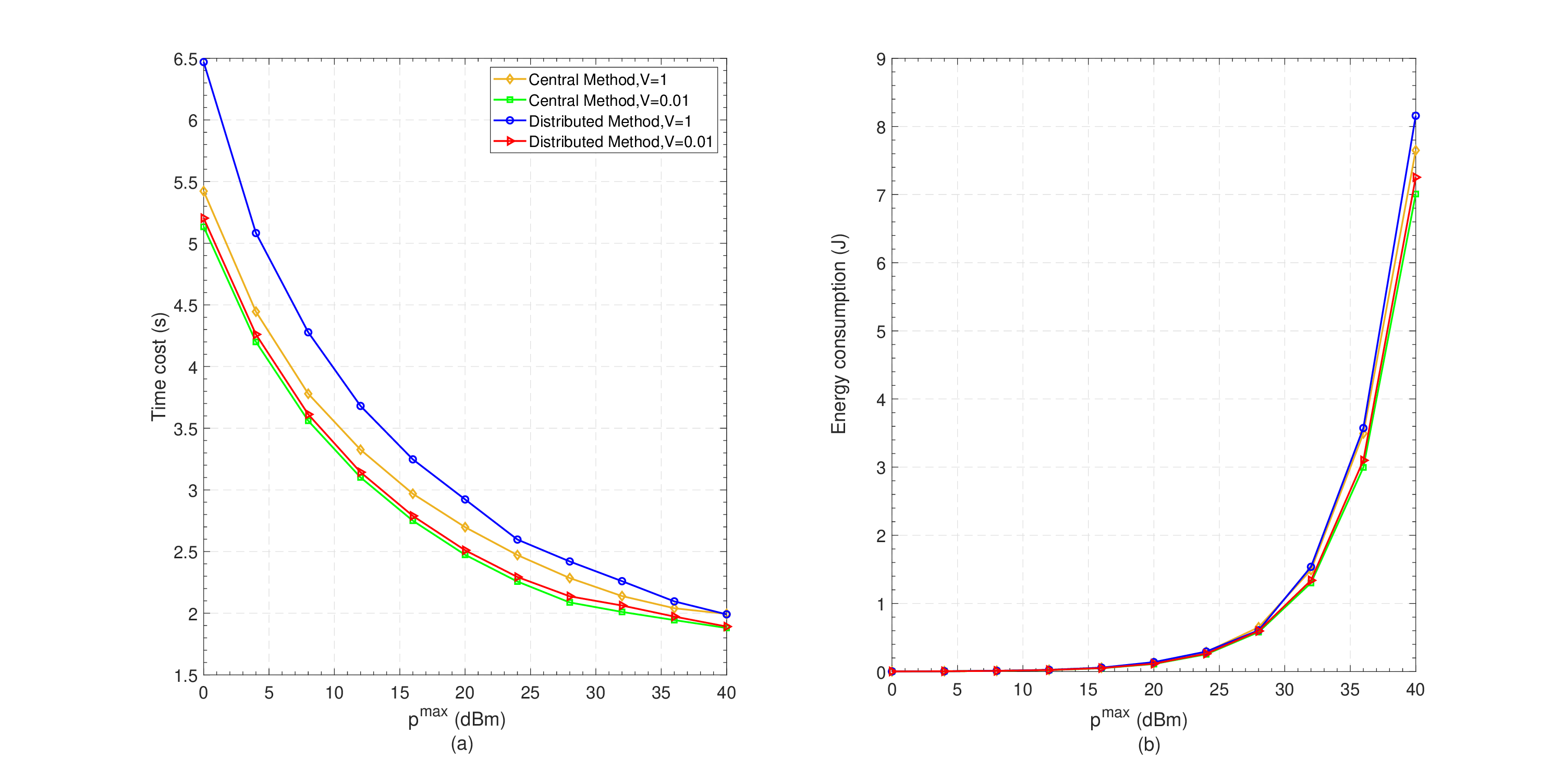}
	\caption{Time-average raining time cost and energy consumption versus the power budget of UnRCs.}
	\label{fig:pmax}
\end{figure}

\subsection{Performance Demonstration of Proposed Methods}
{\bf\em 1) Performance comparison versus  $p_{\text{urc}}^{\max}$: }
Fig. \ref{fig:pmax}(a) and (b) plot the performances of the centralized and distributed methods, respectively, with $V=1$ and $V=0.01$, and $p_{\text{urc}}^{\max}$ varied from $0\text{~dBm}$ to $40\text{~dBm}$.
Once again, we can observe that the centralized method outperforms the distributed method for$V \in \{0.01, 1\}$.
As can be seen from the figures, for both the centralized and distributed methods, the average training time (energy consumption) decreases (increases) for a given value of $V$, as the power budget $p_{\text{urc}}^{\max}$ goes up.
This observation can be explained with the aid of (\ref{l:1}), (\ref{l:2}), and (\ref{l:3}) as follows.
For a given data size,  it can be shown from (\ref{l:2}), (\ref{l:3}) and (\ref{l:4}) that the incurred uploading time monotonically decreases with the increasing maximum transmit power of UnRCs.
By contrast, the trend with respect to energy consumption depicted in Fig. \ref{fig:pmax}(b) is attributed to the exponential increase in transmit power, which dominates the ratio of the power consumption to the achieved data rate and the energy consumption increases more rapidly when $p_{\text{urc}}^{\max}$ is larger than $30\text{~dBm}$.
As can be seen from Fig. \ref{fig:pmax},  it is difficult to distinguish the curves for $V=1$ and $0.01$, when the maximum transmit power is lower than $30\text{~dBm}$.
\begin{figure}[!t]
	\centering
	\includegraphics[width=3.75in]{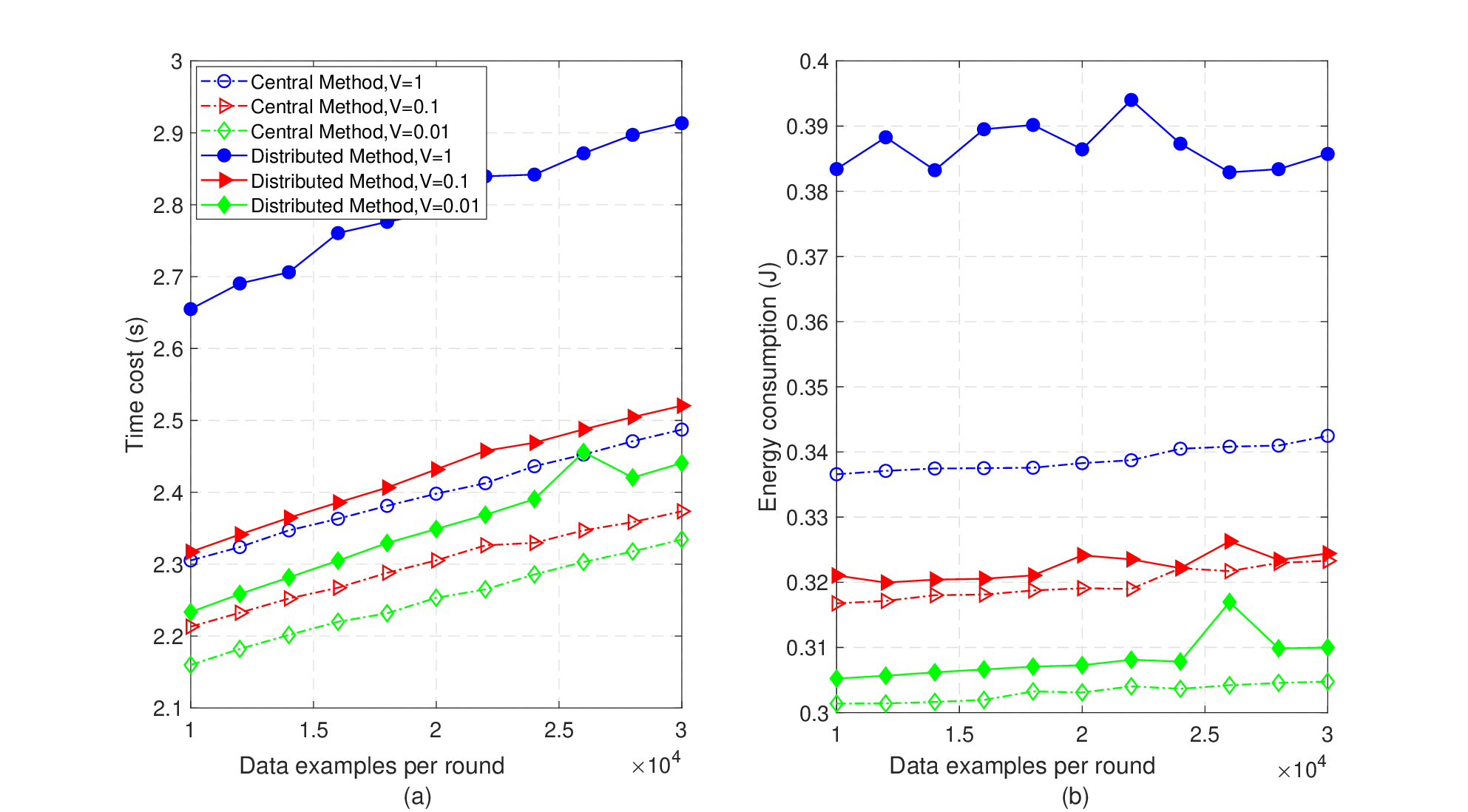}
	\caption{Time-average training time cost and  energy consumption versus data size.}
	\label{fig:datasize}
\end{figure}

{\bf\em 2) Performance comparison versus $Q$: }
Fig. \ref{fig:datasize} presents the performance of the proposed methods with $p_{\text{rc}}^{\max}=0.5\text{~W}$ and $p_{\text{urc}}^{\max}=0.3\text{~W}$, with $V=1,  0.1,$ and $0.01$ and $Q$ varied from $10^4$ to $3\times 10^4$.
As can be observed from the figure,  both the training time in Fig. \ref{fig:datasize}(a) and  the average energy consumption in Fig. \ref{fig:datasize}(b) increase as the data size of the clients goes up, since the data size is always detrimental to the unilateral system efficiency.
According to (\ref{l:6}) and (\ref{l:7}), the training time and energy consumption are almost linearly with regards to the data size.
Furthermore, it is noted that for the proposed two methods both the training time  and energy consumption performance decrease with the increase of the value of $V$.
\subsection{Implementation on the MNIST/CIFAR-10 Datasets}
Note that our proposed LRef-FedCS approach for FL in the HieIoT network can be regarded as a pre-selection component integrated with existing federated platforms.
In Figd. \ref{fig:exp1} and \ref{fig:exp}, to show the efficiency of our proposed LRec-FedCS approach, we run a real-world PyTorch model based on two standard datasets MNIST and CIFAR-10.
We train a convolutional neural network (CNN) model with two convolutional layers, a dropout layer, and two linear fully-connected layers on the MNIST dataset, which includes a training set of $60000$ images and a testing set of $10000$ images.
The CNN consists of a convolutional layer, a pooling layer, a convolutional layer, and three linear fully-connected layers and is trained on the CIFAR-10 dataset, which includes $50000$ training color images and $10000$ testing color images.
In our experiments, to simulate different quality levels of local data, we generate different percentages of noisy labels for the clients (RCs and UnRCs) based on the trust values between the RCs and UnRCs.
Specifically, each RC $v_m$ owns $100(1-w_{m,:}/w)\%$ noisy data, where $w_{m,:}=\sum_{v_n\in{\cal N}}w_{m,n}$ and $w=\sum_{v_m\in{\cal M}, v_n\in{\cal N}}w_{m,n}$.
Likewise, each UnRC $v_n$ owns $100(1-w_{:,n}/w)\%$ noisy data with $w_{:,n}=\sum_{v_m\in{\cal M}}w_{m,n}$.
Figs. \ref{fig:exp1} and \ref{fig:exp} draw the global accuracy of the eight methods mentioned in Section VIII.C.
It can be seen from the two figures that our designed centralized and distributed methods outperform the other competing methods other than the greedy method.
However, the aggressive attitude towards global model accuracy exhibited by the greedy method can lead to catastrophic cost increase at the client level, thereby resulting in faster battery depletion and higher cost.
It is clear from Figs. \ref{fig:exp1}-\ref{fig:exp}  and Fig. \ref{fig:compare} that the solutions yielded by our proposed centralized and distributed LRef-FedCS methods with trust-aided resource scheduling is able to balance well between learning quality and cost.
\begin{figure}[!t]
\begin{minipage}[t]{0.235\textwidth}
\centering
\includegraphics[width=1\textwidth]{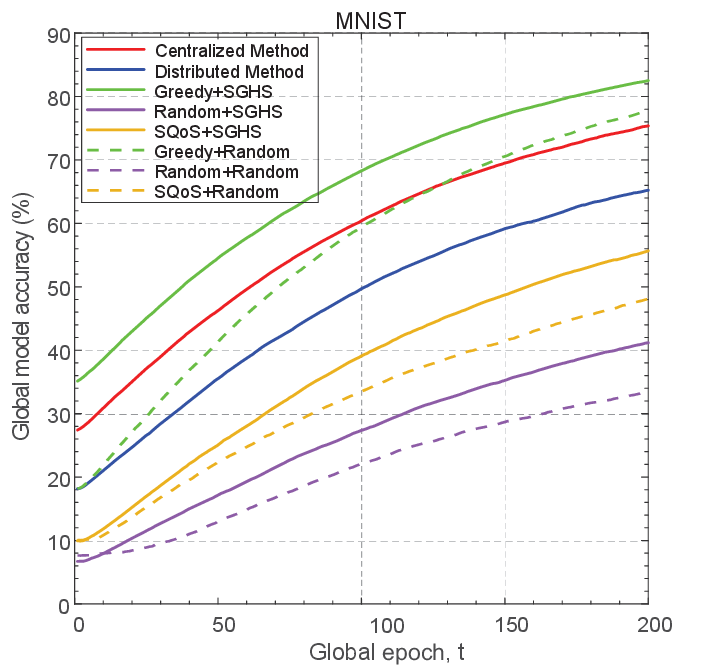}
\caption{Global accuracy on MNIST dataset.}
\label{fig:exp1}
\end{minipage}
\begin{minipage}[t]{0.235\textwidth}
\centering
\includegraphics[width=1\textwidth]{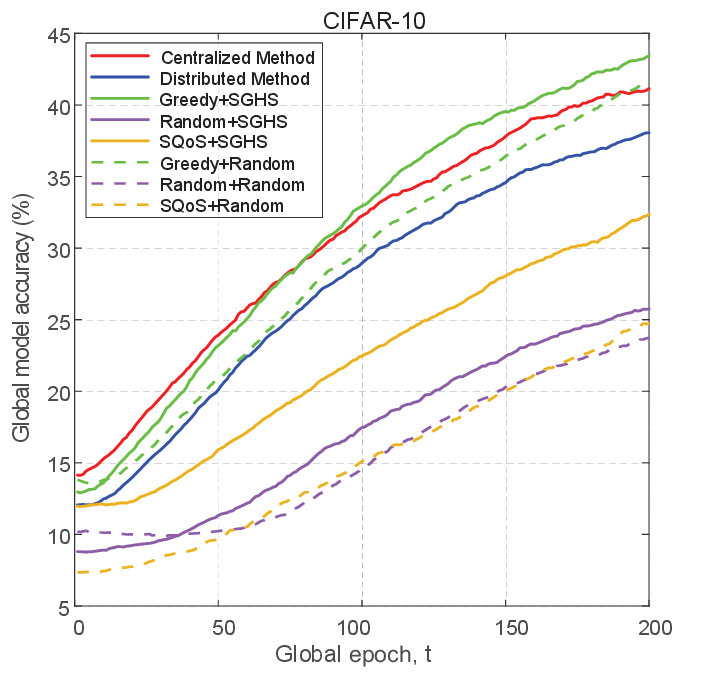}
\caption{Global accuracy on CIFAR-10 dataset.}
\label{fig:exp}
\end{minipage}
\end{figure}
\section{Conclusion}
In this paper,  we proposed the LRef-FedCS approach for FL in the HieIoT network built through social relationships, which yields trustworthy and cost-effective client selection policies.
The problem was formulated into an infinite-horizon renewal problem subject to time-average fairness constraints.
Drawing on the Lyapunov optimization technique, we first transformed the time-coupled problem into a step-by-step SWET-based min-max optimization problem.
To tackle the mixed-integer optimization problem, we proposed two two-stage iterarive approaches, by employing centralized search, matching game, and SGHS algorithm.
In particular, the first algorithm employs specific centralized search for LRef-FedCS, and SGHS for LMAO.
Instead, the second algorithm can realize autonomous client selection via the matching game process.
Simulation results were presented to demonstrate that the proposed LRef-FedCS approach with the designed trust-aided resource scheduling is capable of striking  a good balance between high learning quality and low system cost.

\section*{Appendix A}
We characterize the performance of the stable solution produced by Algorithm \ref{algorithm:3} in each time slot $t$.
As the objective function of (\ref{add:32}) is used to minimize the maximum utility associated with all RCs, the outcome of Algorithm \ref{algorithm:3} in each iteration $r$ of each time slot $t$ can be mapped to the joint LRef-FedCS and local model accuracy optimization vector $(\tilde{\boldsymbol a}, \theta)$ as $({\cal MG}_r[t], \theta_r[t])\mapsto (\tilde{\boldsymbol a}_r[t], \theta_r[t]).$
For a given $\theta_{r-1}[t]$, based on the accept and reject operation at each iteration $r$ of each time slot $t$ in Algorithm \ref{algorithm:3}, the matching ${\cal MG}$ captures the minimum value of $-U_m(\tilde{\boldsymbol a}_r[t], \theta_{r-1}[t])$ among all RCs, in which the matching at the $r\text{th}$ iteration guarantees that $\min_{v_m\in{\cal M}} \{-U_m(\tilde{\boldsymbol  a}_r[t], \theta_{r-1}[t])\}\leq \min_{v_m\in{\cal M}} \{-U_m(\tilde{\boldsymbol  a}_{r-1}[t], \theta_{r-1}[t])\}$.
Next, with a large \texttt{NI} for the SGHS algorithm (Stage II) of Algorithm  \ref{algorithm:3}, the objective function of (\ref{add:32}) is monotonically non-increasing after each iteration, mathematically expressed as $F(\theta_r[t])\leq F(\theta_{r-1}[t]).$
Therefore, Algorithm \ref{algorithm:3} is guaranteed to converge to a local solution of the optimization problem (\ref{add:32}).


\end{document}